\documentclass[nohyperref]{article}

\usepackage{microtype}
\usepackage{graphicx}
\usepackage{subfigure}
\usepackage{amsmath,amssymb}
\usepackage{booktabs} %

\usepackage{array,multirow}

\usepackage{tabularx}
\makeatletter
\newcommand{\multiline}[1]{%
  \begin{tabularx}{\dimexpr\linewidth-\ALG@thistlm}[t]{@{}X@{}}
    #1
  \end{tabularx}
}
\makeatother

\usepackage[accepted]{icml2022}

\usepackage{hyperref}

\usepackage{amsmath}
\usepackage{amssymb}
\usepackage{mathtools}
\usepackage{amsthm}
\usepackage{natbib}

\usepackage{shortcuts}

\usepackage[capitalize,noabbrev]{cleveref}

\usepackage{thmtools} %
\theoremstyle{plain}
\newtheorem{theorem}{Theorem}[section]

\newtheorem{lemma}[theorem]{Lemma}
\crefname{lemma}{Lemma}{Lemmas}

\theoremstyle{definition}
\newtheorem{definition}[theorem]{Definition}
\newtheorem{assumption}[theorem]{Assumption}
\crefname{assumption}{Assumption}{Assumptions}
\theoremstyle{remark}
\newtheorem{remark}[theorem]{Remark}

\usepackage{nicefrac}       %
\usepackage{enumitem}
\usepackage{tikz}
\usetikzlibrary{decorations.pathreplacing}

\newcommand{\alphalb}{\underline{\alpha}}
\newcommand{\alphaub}{\overline{\alpha}}
\newcommand{\supoveralpha}{\sup_{\alpha \in [\alphalb, \alphaub]}}
\newcommand{\supoverpi}{\sup_{\pi \in \Pi}}
\newcommand{\piR}{R(\pi(S))}
\newcommand{\pib}{\pi_0}
\newcommand{\est}[1]{\widehat{#1}}

\newcommand{\droValue}{\mV_\delta}
\newcommand{\droRegret}[1]{\mR_\delta\prns{#1}}

\newcommand{\piblb}{\eta}
\newcommand{\rewardSupportLb}{\omega}
\newcommand{\ipsMean}{S_w}
\newcommand{\minRewardIpsMean}{S_w^m}
    
\newcommand{\InitialEstimate}{\operatorname{InitialEstimate}}

\newcommand{\nmP}[1]{\nm{#1}_{L_2\prns{\PP_0}}}
\newcommand{\nmPn}[1]{\nm{#1}_{L_2\prns{\PP_N}}}

\newcommand{\wlb}{\underbar{W}}
\newcommand{\distHamming}[2]{d_H\prns{#1, #2}}

\newcommand{\workOf}[2]{\mT_{#1}\prns{#2}}
\newcommand{\spanOf}[2]{\mS_{#1}\prns{#2}}

\icmltitlerunning{Doubly Robust Distributionally Robust Off-Policy Evaluation and Learning}

\begin{document}

\twocolumn[
\icmltitle{Doubly Robust Distributionally Robust Off-Policy Evaluation and Learning}

\icmlsetsymbol{equal}{*}

\begin{icmlauthorlist}
\icmlauthor{Nathan Kallus}{equal,cu}
\icmlauthor{Xiaojie Mao}{equal,th}
\icmlauthor{Kaiwen Wang}{equal,cu}
\icmlauthor{Zhengyuan Zhou}{equal,nyu}
\end{icmlauthorlist}

\icmlaffiliation{cu}{Cornell University and Cornell Tech}
\icmlaffiliation{th}{Tsinghua University}
\icmlaffiliation{nyu}{Arena Technologies and New York University}

\icmlcorrespondingauthor{Kaiwen Wang}{\href{https://kaiwenw.github.io/}{\nolinkurl{https://kaiwenw.github.io}}}

\icmlkeywords{Distributional Robustness, Double Robustness, Off-Policy Evaluation, Off-Policy Learning}

\vskip 0.3in
]

\printAffiliationsAndNotice{\icmlEqualContribution} %

\begin{abstract}
Off-policy evaluation and learning (OPE/L) use offline observational data to make better decisions, which is crucial in applications where online experimentation is limited.
However, depending entirely on logged data, OPE/L is sensitive to environment distribution shifts --- discrepancies between the data-generating environment and that where policies are deployed. \citet{si2020distributional} proposed distributionally robust OPE/L (DROPE/L) to address this, but the proposal relies on inverse-propensity weighting, whose estimation error and regret will deteriorate if propensities are nonparametrically estimated and whose variance is suboptimal even if not.
For standard, non-robust, OPE/L, this is solved by doubly robust (DR) methods, but they do not naturally extend to the more complex DROPE/L, which involves a worst-case expectation.
In this paper, we propose the first DR algorithms for DROPE/L with KL-divergence uncertainty sets.
For evaluation, we propose \textbf{L}ocalized \textbf{D}oubly \textbf{R}obust \textbf{DROPE} (LDR$^2$OPE) and show that it achieves semiparametric efficiency under weak product rates conditions.
Thanks to a localization technique, LDR$^2$OPE only requires fitting a small number of regressions, just like DR methods for standard OPE.
For learning, we propose \textbf{C}ontinuum \textbf{D}oubly \textbf{R}obust \textbf{DROPL} (CDR$^2$OPL) and show that, under a product rate condition involving a continuum of regressions, it enjoys a fast regret rate of $\mO(N^{-1/2})$ even when unknown propensities are nonparametrically estimated.
We empirically validate our algorithms in simulations and further extend our results to general $f$-divergence uncertainty sets.
\end{abstract}

\section{Introduction}
The vast majority of online recommendations in search engines, e-commerce, social media, streaming platforms, etc. are made by algorithms that learn from historical user interactions \citep{li2010contextual,JMLR:v14:bottou13a,ren2020dynamic, liu2021reinforcement}.
Even in high-stakes domains, such as healthcare \citep{murphy2003optimal} and education \citep{mandel2014offline}, the promise of cheaper and higher quality decisions, made possible by the growing abundance of user-specific data, incentivize the inclusion of automatic decision-making components into existing approaches.

This task of making good decisions from observational data is formalized by the problems of off-policy evaluation (OPE) \citep{foster2019orthogonal,kallus2020double,chernozhukov2018double,farajtabar2018more,joachims2016counterfactual,JMLR:v14:bottou13a,dudik2011doubly} and off-policy learning (OPL) \citep{manski2004statistical,kitagawa2018should, athey2021policy, zhan2021policy, zhou2018offline, kallus2020statistically,swaminathan2015counterfactual,dudik2011doubly}.
OPE is concerned with estimating the expected returns of a target policy given logged data, collected under a different behavior policy. OPL is concerned with learning a policy that maximizes the expected returns given this data.
OPE/L assumes that the the environment in which these policies are deployed is identical to the environment that generated the training data.
In practice, this often is not the case.
For example, in recommendation systems, user interests naturally shift with seasonality and world events, which correspond to changes in the state and reward distributions. Moreover, the environment could also be adversarially perturbed by attackers or data corruption.

Distributional robustness is a way to guard against such unknown discrepancies between training and deployment environments. Instead of estimating/maximizing the expected policy return under the training environment, we may consider estimating/maximizing the worst-case return over all environments within an uncertainty set around the unknown training environment.
\citet{si2020distributional, si2020distributionally} tackle this distributionally robust OPE/L (DROPE/L) problem using methods based on self-normalized inverse propensity scoring (SNIPS) \citep{swaminathan2015self}. The uncertainty sets of \citep{si2020distributional} and this paper are with respect to the KL-divergence, and generally $f$-divergences.

However, \citep{si2020distributional} assumes that we know the behavior propensities, which are usually absent in observational datasets.
One may consider simply fitting and imputing the propensities using some flexible machine learning (ML) methods, i.e. non-parametric estimators of nuisance functions.
As the propensity estimates may converge at slow rates, this leads to slow rates in estimation and learning for the proposed SNIPS-based methods.
Even with known propensities, the SNIPS-based estimator's asymptotic variance for DROPE is in fact suboptimal.

In standard (non-distributionally robust) OPE/L, doubly robust (DR) is the canonical approach for improving estimation variance and for alleviating the sensitivity to estimation of nuisances, i.e. unknown functions such as propensities.
In addition to fitting a propensity model, DR also fits the expected reward given state and action and combines the two models to construct an estimator with better statistical properties.
A key result in OPE is that the cross-fitted DR estimator (CFDR) is $\sqrt{N}$-consistent, asymptotically linear and efficient (i.e. attains the lowest possible asymptotic variance), even when nuisances are estimated at slower-than-$\sqrt{N}$-rates \citep{chernozhukov2018double}.
This, however, does not immediately extend to DROPE/L, whose objective is formed as a supremum over the log of moment generating functions. It therefore remains a question how to obtain estimation-robustness guarantees for DROPE/L.

In this paper, we propose novel doubly robust algorithms for DROPE/L, ensuring robustness to \emph{both} environment shifts and model estimation errors.
Our contributions are summarized as follows:
\begin{enumerate}
	\item For DROPE, we propose the Localized DR DROPE (LDR$^2$OPE) estimator and show that it is $\sqrt{N}$-consistent, asymptotically linear, and enjoys semiparametric efficiency under weak product rates (\cref{sec:localization}). In particular, just like DR estimators for standard OPE, LDR$^2$OPE only requires fitting a few regressions, including a propensity and two transformed-outcome regressions.
	\item For DROPL, we propose Continuum DR DROPL (CDR$^2$OPL) and prove a $\mO(N^{-1/2})$ regret guarantee, even when propensities are nonparametrically estimated at slow rates (\cref{sec:doubly_robust_dropl}). 
	\item We empirically show that our proposals outperform benchmarks in simulation (\cref{sec:decision_tree_policy_experiments}). Code is available at \href{https://github.com/CausalML/doubly-robust-dropel}{\nolinkurl{https://github.com/CausalML/doubly-robust-dropel}}.
	\item We further extend our methods to general $f$-divergence uncertainty sets (\cref{sec: f-divergence}).
\end{enumerate}

\subsection{Related Literature}
We work in the distributionally robust setting (\cref{sec:prelim-distributional-robust}) proposed by \citet{si2020distributional}, which was motivated by the distributionally robust optimization (DRO) literature \citep[\textit{e.g.}, ][]{hu2013kullback,ben2013robust}.
Unlike \citet{si2020distributional}, 
\textit{we do not assume that the behavior policy is known.}
To derive our doubly robust DROPE estimator, we propose a novel formulation of the DRO problem as a multidimensional moment equation and leverage the techniques of \citet{kallus2019localized}. 
This allows us to tackle the complex optimization formulation of the objective and still attain semiparametric efficiency under very lax conditions.

In standard (non-distributionally robust) OPL, maximizing the CFDR objective was shown to have $\mO(N^{-1/2})$ regret even under slow nuisance estimator 
by arguing the CFDR objective concentrates uniformly over a policy class of bounded complexity \citep{zhou2018offline,athey2021policy}.
However, this result is for standard policy learning, without environment shifts. As such, their uniform concentration results are with respect to the best policy in the training environment.
In our setting, we are aiming to learn the best policy in the worst-case testing environment, which is a different formulation and requires a new set of techniques.
In particular, we show that our objective concentrates uniformly not only over policies, but also over all adversarial environments, yielding our $\mO(N^{-1/2})$ distributionally robust regret guarantee.
\citet{si2020distributional} also proved a $\mO(N^{-1/2})$ distributionally robust regret bound but they crucially assumed known propensities, which allowed them to estimate the DROPL objective by reweighting via SNIPS. Also, their proof strategy is different for discrete and continuous rewards, and is specialized for SNIPS. In contrast, our proof directly decomposes the DRO objective, handling all cases in a unified way.

\citet{si2020distributional} do not discuss why self-normalization (SN) was used (as opposed to IPS without SN), and only referenced \citet{swaminathan2015self}, which proposed SNIPS for non-distributionally-robust OPE/L. As an aside, we show in \cref{sec:weighted_estimators_degenerate_drope} that even though non-normalized IPS is in fact theoretically well-behaved for standard OPE under overlap conditions, the unique structure of DROPE renders IPS degenerate even under such conditions, which highlight the unique importance of SN in the DRO setting.

\citet{mo2021learning,liu2019triply} studied distributionally robust learning in the context of state distribution shifts (covariate shift). \citet{kido2022distributionally} studied distributionally robust learning in the context of known covariate shift and unknown outcome distribution shift (concept shift), under the Wasserstein distance. 
We highlight that these problems, and 
the meaning of policy value/regret therein, 
are different from our setting, as we study \emph{unknown} covariate and \emph{unknown} concept shifts, under the KL-divergence and $f$-divergences.

\section{Preliminaries}\label{sec:offline_dro_cb_prelim}
We use the standard data generation process of OPE/L. 
Our data $\mD = \braces{(s_i, a_i, r_i)}_{i \in [N]}$ consists of $N$ i.i.d. draws of $(S, A, R)$ generated as follows.
The state and potential outcomes $(S, R(a^1),...,R(a^{\abs{\mA}})) \in \mS \times [0, 1]^{\abs{\mA}}$ are drawn from the nominal environment $\P_0$, where $\mS$ is the state space, $\mA$ is the discrete action space, and $R(a)$ denotes the potential reward from taking an action $a$ \citep{neyman1923applications, rubin1974estimating}. 
An \textit{unknown} behavior policy $\pib$ then samples an action $A\sim\pib(S)$ given the observed state, i.e. $A=a$ with probability $\pi(a\mid s)$. Out of the potential outcomes, only the factual outcome corresponding to the chosen action $R = R(A)$ is observed. 

For a (stochastic) policy $\pi$, we use $\piR$ to denote the random reward corresponding to the action sampled from $\pi$. Unless stated otherwise, $\mathbb{E}$ and $\PP$ are taken over $\PP_0$.

\begin{assumption}\label{asm:standard_cb}
We posit standard assumptions from the OPE/L literature \citep{si2020distributional}:
\begin{enumerate}[label=(\roman*)]
    \item Unconfoundedness: 
    \(
    (R(a^1), \ldots, R(a^{\abs{\mA}})) \indep A \mid S
    \).
    \item Strong overlap: $\piblb := \inf_{s \in \mS, a \in \mA} \pib(a \mid s) > 0$.
\end{enumerate}
Furthermore, there exists $\rewardSupportLb > 0$ such that,
\begin{enumerate}
    \item If $R(a) \mid S$ is continuous, its PDF $p_R(r\mid s,a)$ is lower bounded: $p_R(r\mid s,a) \geq \rewardSupportLb, \forall r \in [0, 1]$.
    \item If $R(a) \mid S$ is discrete, its PMF $p_R(r\mid s,a)$ is lower bounded: $p_R(r\mid s,a) \geq \rewardSupportLb, \forall r \in \mathbb{D}$, where $\mathbb{D}$ is the set of possible rewards and WLOG $0 \in \mathbb{D}$.
\end{enumerate}
\end{assumption}
More generally, we may require $R(a)\mid S=s$ to be mutually absolutely continuous with respect to a common measure on $[0,1]$ for almost all states $s\in\mS$.

\subsection{Distributionally Robust Formulation of OPE/L} \label{sec:prelim-distributional-robust}
We now recall the KL-distributionally robust formulation of OPE/L due to \citet{si2020distributional}. For an alternative environment $\PP_1$, the KL-divergence is a notion of how different $\PP_1$ is from $\PP_0$ and is defined as $D_{KL}(\PP_1 \Mid \PP_0) = \E_{\PP_1}\bracks{ \log\prns{\frac{\diff\PP_1}{\diff\PP_0}}}$. 
Let $\delta > 0$ denote the magnitude of distribution shifts we seek to be robust to, which we take as a fixed hyperparameter.
Define the uncertainty set $\mU(\delta) = \braces{\PP_1: \PP_1 \ll \PP_0 \wedge D_{KL}(\PP_1 \Mid \PP_0) \leq \delta}$ to be the set of perturbed environments $\PP_1$ which are $\delta$-close to the nominal distribution $\PP_0$, as measured by the KL-divergence.
We highlight that \emph{both} the state and reward distributions can be perturbed. 
For a policy $\pi$, the distributionally robust value $\droValue(\pi)$ is its worst-case performance under environment shifts with magnitude at most $\delta$, formalized as follows.
\begin{align}
    \droValue(\pi) \defeq \inf_{\PP_1 \in \mU(\delta)} \E_{\PP_1}\bracks{\piR}  \label{eq:def_phi_infinite_dimensional}
\end{align}
We remark that there are data-driven, calibration methods to choose $\delta$, e.g. \citep{mo2021learning}.

This leads to the definitions of distributionally robust off-policy evaluation and learning (DROPE/L):
\begin{enumerate}[leftmargin=1.45cm]
    \item[\tb{DROPE:}] For a policy $\pi$ and radius $\delta>0$, estimate the worst-case value $\droValue(\pi)$.
    \item[\tb{DROPL:}] For a policy class $\Pi$ and radius $\delta>0$, find a near-optimally robust policy $\wh\pi \in \Pi$ with small regret in worst-case policy value $$\droRegret{\pi} \defeq \droValue(\pi\opt) - \droValue(\pi),$$ where 
    $\pi\opt \in \argmax_{\pi \in \Pi} \droValue(\pi)$.
\end{enumerate}

While the infinite-dimensional infimum in \cref{eq:def_phi_infinite_dimensional} seems intractable, it is in fact equivalent to a supremum over a dual variable $\alpha$. We now recall this strong duality result from \citet[Lemmas 1 and A11]{si2020distributional}.
\begin{lemma}
Suppose \cref{asm:standard_cb}. 
The distributionally robust value $\droValue(\pi)$ defined in \cref{eq:def_phi_infinite_dimensional} is equivalent to,
\begin{align}
    \droValue(\pi) = \max_{\alpha > 0}\phi(\pi,\alpha) 
    &\defeq -\alpha \log W(\pi,\alpha) - \alpha \delta \label{eq:def_phi} \\
    \text{where } W(\pi,\alpha) &\defeq \E\bracks{\exp(-\piR/\alpha)}. \label{eq:def_W}
\end{align}
Furthermore, $\phi(\pi,\cdot)$ is strictly concave, and is maximized at a unique $\alpha\opt(\pi) \in (0, \alphaub]$, where $\alphaub \coloneqq 1/\delta$.
\end{lemma}

In particular, for any policy, we know that $\alpha\opt(\pi) > 0$. For our DROPL analysis, we need this lower bound to hold uniformly over $\Pi$, stated in the following assumption. 
\begin{assumption}\label{asm:alpha_bounded_from_zero}
$\alphalb \defeq \inf_{\pi \in \Pi} \alpha\opt(\pi) > 0$. 
\end{assumption}
We also denote $\wlb \coloneqq \rewardSupportLb \frac{\min(\alphalb, 1)}{2}$ if $R(a) \mid S$ is continuous, and $\wlb \coloneqq \rewardSupportLb$ if discrete. \cref{lm:w_lower_bound} shows that $W(\pi,\alpha) \geq \wlb$ for any $\pi\in\Pi$ and any $\alpha \geq \alphalb$.

\citet{si2020distributional} assume that the behavior policy $\pib$ is known, and propose to estimate $W(\pi,\alpha)$ with SNIPS,
based on normalizing the propensity ratios $w_i = \frac{\pi(a_i \mid s_i)}{\pib(a_i \mid s_i)}$:
\begin{align*}
    \wh W^{SNIPS}(\pi,\alpha) = \sum_{i=1}^N \frac{w_i}{\sum_j w_j} \exp(-r_i/\alpha).
\end{align*}
Plugging $\wh W^{SNIPS}(\pi,\alpha)$ into \cref{eq:def_phi} gives an estimator for the robust policy value. Assuming that the behavior policy $\pib$ is known, \citet{si2020distributional} show that the resulting estimator is a $\sqrt{N}$-consistent for DROPE, and the resulting DROPL algorithm can achieve $\mO(N^{-1/2})$ regret guarantee. However, in practice, the behavior policy is often unknown and needs to be estimated (often at slow rates). Moreover, inverse-propensity scoring and its self-normalized variant cannot achieve semiparametric efficiency. This motivates us to consider improved doubly robust methods.

\section{Doubly Robust DROPE}
To estimate the robust policy value $\droValue(\pi)$ in a  doubly robust way, it is natural to first consider estimating 
$W(\pi,\alpha)$ in \cref{eq:def_W} with a doubly robust estimator. This however requires estimating a continuum of regression functions $\braces{f_0(\cdot, \cdot; {\alpha}): \mS \times \mA \mapsto \R{}:  0 < \alpha \le \alphaub}$, where 
\begin{align}
    f_0(s,a; \alpha) \defeq \Eb{\exp(-R/\alpha) \mid S=s,A=a}, \label{eq:cfdr_outcome_target}
\end{align}
is parameterized by the dual variable $\alpha$.
This means that we would need to fit a large number or even {infinitely} many regressions functions.
This is in stark contrast to standard OPE where doubly robust estimation requires fitting only a single regression function $\Eb{R \mid S=s, A=a}$.

To overcome the challenge of fitting a continuum of regressions, we propose to  leverage the Localized Debiased Machine Learning (LDML) framework recently developed for causal inference \citep{kallus2019localized}.
To do so, we cast the estimation of $\alpha\opt(\pi)$ and $\droValue(\pi)$ into a joint moment estimation problem.
We then develop a localized doubly robust algorithm that only fits two regressions at an initial estimate of $\alpha\opt(\pi)$, instead of infinitely many regressions.

\subsection{The Localization Approach}
\label{sec:localization}
First, since that $\phi(\pi,\cdot)$ is strictly concave (\cref{lm:phi_concavity}), observe that $\alpha\opt := \alpha\opt(\pi)$ is the unique root to $\partialwrt{\alpha} \phi(\pi, \alpha) = 0$, and satisfies
\begin{align}
&-\log W_0(\pi, \alpha\opt) - \frac{W_1(\pi,\alpha\opt)}{\alpha\opt W_0(\pi,\alpha\opt)} - \delta = 0  \label{eq:dro_eval_deriv_zero} \\
\text{where }
&W_j(\pi,\alpha) \defeq \Eb{\piR^j \exp(-\piR/\alpha)}. \nonumber
\end{align}
Moreover, we know from \cref{eq:def_phi} that
\begin{align}\label{eq: VDRO2}
    \droValue(\pi) =  -\alpha\opt \log W_0\opt - \alpha\opt \delta,
\end{align}
where we use the shorthand $W_j\opt = W_j(\pi,\alpha\opt)$.
Therefore, estimating $\alpha\opt(\pi)$ and $\droValue(\pi)$ in \cref{eq:dro_eval_deriv_zero,eq: VDRO2} is equivalent to estimating the root of the following moment equation with parameter $\theta = [\alpha, W_0, W_1, \droValue]^\top$:
\begin{align}
	\Eb{ U(\piR; \alpha) + V(\theta) } = \tb{0} \label{eq:def_drope_uv}
\end{align}
\resizebox{1.05\linewidth}{!}{
\(
	U(r; \alpha) = 
	\begin{bmatrix}
		\exp(-r / \alpha) \\
		r \exp(-r / \alpha) \\
		0 \\
		0 
	\end{bmatrix},
	V(\theta) = 
	\begin{bmatrix}
		-W_0 \\
		-W_1 \\
		-\delta - \log W_0 - \frac{W_1}{\alpha W_0} \\
		-\droValue - \alpha \log W_0 - \alpha \delta
	\end{bmatrix}. \nonumber
\)
}

Since we don't observe the counterfactual $\piR$, \cref{eq:def_drope_uv} is infeasible for estimation.
Instead, we derive the following doubly robust moment equation in terms of the observed variables, with nuisances $\eta_1,\eta_2$ to be estimated:
\begin{align}
	&\Eb{\psi(Z; \theta, \eta_1\opt(Z; \alpha), \eta_2\opt(Z))} = \tb{0} \label{eq:ortho_moment_drope} \\
	&\psi(z; \theta, \eta_1(z; \alpha), \eta_2(z)) = \frac{\pi(a \mid s)}{\eta_2(s, a)} \prns{U(r; \alpha) - \eta_1(s, a; \alpha)} \nonumber \\
	&+ \Eb[a \sim \pi(s)]{\eta_1(s, a; \alpha)} + V(\theta), \nonumber
\end{align}
where $\eta_2\opt(z) = \pib(a \mid s)$ is the behavior propensity and
\begin{align*}
    &\eta_1\opt(s, a; \alpha) = \E[U(R; \alpha) \mid S =s, A= a] \\
    &\phantom{\eta_1\opt(s, a; \alpha)} = [  f_0(s,a;\alpha), f_1(s,a;\alpha), 0, 0]^\top, \\
    &f_j(s, a; \alpha) \defeq \Eb{R^j \exp(-R/\alpha) \mid S=s,A=a}.
\end{align*}
Importantly, \cref{eq:ortho_moment_drope} involves not only the regression function $f_0$ in \cref{eq:cfdr_outcome_target}, but also an additional regression function $f_1$. With this new regression function,  the G\^{a}teaux derivatives of $\Eb{\psi(Z; \theta, \eta_1(Z; \alpha), \eta_2(Z))}$ with respect to the functions $(\eta_1, \eta_2)$ are zero when evaluated at $\theta\opt=\prns{\alpha\opt, W_0\opt, W_1\opt, V_{\delta}(\pi)}$, $\eta_1(\cdot; \alpha) = \eta_1\opt(\cdot; \alpha\opt)$, and   $\eta_2 = \eta_2\opt$. This property is  called \emph{Neyman Orthogonality} \citep{chernozhukov2018double}, which implies that the doubly robust moment estimation is insensitive to errors of estimating $\eta_1\opt, \eta_2\opt$.
Therefore, if an initial guess
$\wh\alpha_{init}$ is close enough to $\alpha\opt$, it suffices to only fit  $\eta_1\opt(\cdot; \alpha)$ localized at $\alpha = \wh\alpha_{init}$, rather than the whole continuum of regressions.

\begin{algorithm}[t!]
	\caption{Localized Doubly Robust DROPE}
	\label{algo:dro_eval_ldml}
	\begin{algorithmic}[1]
		\State \textbf{Input:} Data $\mD$, policy $\pi$, uncertainty set radius $\delta$.
		\State Randomly split $\mD$ into $K$ (approximately) even folds, with the indices of the $k^{\text{th}}$ fold denoted as $\mI_k$.
		\For{$k=1,...,K$}  \label{line:ldr_forloop_crossfit}
    		\State Using $\mD[\mI_k^C]$, train $\wh\pib^{(k)}$ to fit $\pib$.
    		\State Randomly split $\mI_k^C$ into two halves $\mJ_1, \mJ_2$.  \label{line:drope_ldr_split_two_halves}
    		\State $\wh\alpha^{(k)}_{init} \gets \InitialEstimate(\mD[\mJ_1], \delta, \pi)$. \label{line:dro_eval_ldml_initial_estimate}
    		\State Using $\mD[\mJ_2]$, train $\wh f_j^{(k)}$ to fit $f_j(\cdot; \wh\alpha^{(k)}_{init}), j=0,1$. \label{line:dro_eval_ldml_train_f}
		\EndFor
		\State Find $\est{\alpha}>0$ that solves the estimated moment equation: \label{line:doubly-robust-moment-equation}
		\begin{align*}
		    &- \log(\est{W_0}(\alpha)) - \frac{\est{W_1}(\alpha)}{\alpha \cdot \est{W_0}(\alpha)} - \delta = 0 \qquad \text{where, } \\
			&\est{W_j}(\alpha) \defeq \frac{1}{N} \sum_{k=1}^K \sum_{i \in \mI_k} \est{W_j}^{(i,k)}(\alpha) \\
			&\est{W_j}^{(i,k)}(\alpha)
			\defeq
			\sum_{a\in\mA}\pi(a\mid s_i)
			{\wh f_j^{(k)}(s_i,a)} 
			\\&+ \frac{\pi(a_i \mid s_i)}{\wh \pib^{(k)}(a_i\mid s_i)} \prns{r_i^j \exp(-r_i/\alpha) - \wh f_j^{(k)}(s_i,a_i)}.
		\end{align*}
		\State Calculate $\est{\droValue} \gets -\est{\alpha} \log \est{W_0}(\est{\alpha}) - \est{\alpha}\delta.$ 
		\State \textbf{Return: } $\wh\theta^{\op{LDR}^2\op{OPE}} = \prns{\wh\alpha, \wh W_0(\wh\alpha), \wh W_1(\wh\alpha), \wh\droValue}.$
	\end{algorithmic}
\end{algorithm}

We propose \textbf{L}ocalized \textbf{D}oubly \textbf{R}obust \textbf{DROPE} (LDR$^2$OPE) in \cref{algo:dro_eval_ldml}.
Following LDML \citep{kallus2019localized}, we employ a two-level cross-fitting scheme to accommodate flexible (non-parametric) ML estimators while preserving strong theoretical guarantees.
For each data fold $k \in [K]$, we use the out-of-fold (OOF) data to fit the estimator $\wh\pib^{(k)}$ for $\pib$, and half of OOF data to fit the estimator $\wh f_j^{(k)}$, localized at an estimate $\wh\alpha^{(k)}_{init}$ for $\alpha\opt$ based on the other half of OOF data. These estimators trained on OOF data are then evaluated at data in each corresponding fold, forming the estimated doubly robust moment equation in \cref{line:doubly-robust-moment-equation} of the algorithm.
The final moment equation can be solved with 1D Newton-Raphson with projection to $\Rl^+$ (see \cref{sec:newtons_dro_eval}).
A reasonable candidate for $\InitialEstimate$ is the cross-fitted SNIPS estimator.
Thus, \cref{algo:dro_eval_ldml} only requires fitting propensities and two regression functions; all three regressions are amenable to flexible, black-box ML tools.

\subsection{Asymptotic Theory}\label{sec: localization-theory}
Define the estimation rates $\rho_f, \rho_{\pib}, \rho_\alpha$ as random quantities corresponding to the $L_2$ loss as follows:
\begin{align*}
	&\max_{j=0,1}\nmP{ \wh f_j^{(k)} - f_j(\cdot; \wh\alpha_{init}^{(k)}) } \leq \rho_f(N), \\
	&\nmP{ \wh\pib^{(k)} - \pib } \leq \rho_{\pib}(N), ~~ \abs{ \wh\alpha_{init}^{(k)} - \alpha\opt} \leq \rho_\alpha(N).
\end{align*}
\begin{assumption}[Product Rates for LDR$^2$OPE]\label{asm:product_rates_for_ldr}
We assume that $\rho_{\pib}(N) \cdot \prns{ \rho_f(N) + \rho_\alpha(N) } = o_p(N^{-1/2})$.
\end{assumption}

We now state our main result for DROPE: the asymptotic behavior and optimality of LDR$^2$OPE. Specifically, we show that LDR$^2$OPE converges at a $\mO_p(N^{-1/2})$ rate, i.e. $\sqrt{N}$-consistency, and is asymptotically linear. Furthermore, LDR$^2$OPE achieves semi-parametric efficiency, as its asymptotic variance is the smallest possible variance amongst regular estimators -- equivalently, LDR$^2$OPE is locally minimax optimal in mean-squared error amongst \emph{all} estimators. In essence, this shows that our estimator is asymptotically optimal and amenable to uncertainty quantification with confidence intervals.
\begin{restatable}{theorem}{efficiency}\label{thm:dro_eval_ldml_efficiency}
Suppose \cref{asm:standard_cb,asm:product_rates_for_ldr}.
Let $\theta^\star = [\alpha^\star, W_0^\star, W_1^\star, \droValue^\star]^\top$ be the solution to \cref{eq:def_drope_uv}. 
Then, 
\begin{align*}
	\sqrt{N} (\wh\theta^{\op{LDR}^2\op{OPE}} - \theta\opt)
	&= \frac{1}{\sqrt{N}} \sum_{i=1}^N J^{\star -1} \psi\opt(Z_i) + o_p(1)
\end{align*}
where $Z_i=(s_i,a_i,r_i)$, $\psi$ is defined in \cref{eq:ortho_moment_drope}, $\psi\opt(Z) \defeq \psi(Z; \theta\opt, \eta_1\opt(Z; \theta_1\opt), \eta_2\opt(Z))$, 
\begin{align*}
	J\opt = 
	\begin{bmatrix}
		\frac{W_1\opt}{(\alpha\opt)^2} & -1 & 0 & 0 \\
		\frac{W_2\opt}{(\alpha\opt)^2} & 0 & -1 & 0 \\
		\frac{W_1\opt}{(\alpha\opt)^2 W_0\opt} & -\frac{1}{W_0\opt} + \frac{W_1\opt}{\alpha\opt (W_0\opt)^2} & -\frac{1}{\alpha\opt W_0\opt} & 0 \\
		-\log W_0\opt - \delta & -\frac{\alpha\opt}{W_0\opt} & 0 & -1
	\end{bmatrix}
\end{align*}
and $\Sigma = \Eb{J^{\star -1} \psi\opt(Z) \psi\opt(Z)\tr J^{\star -\intercal} }$ is the optimal covariance.
Hence, $\sqrt{N} (\wh\theta^{\op{LDR}^2\op{OPE}} - \theta\opt)\rightsquigarrow \mN(0, \Sigma)$ and
$\hat\theta^{\op{LDR}^2\op{OPE}}$ achieves the semiparametric efficiency lower bound for $\theta\opt$. 
\end{restatable}
Please see \cref{sec:efficiency-proof} for the proof.
\cref{asm:product_rates_for_ldr} is the product rate condition, which has the desired multiplicative structure that allows trading off estimation rates between nuisances. If $\InitialEstimate$ is cross-fitted SNIPS, then Proposition 2 of \citet{kallus2019localized} implies that $\rho_\alpha(N) = \mO_p(\rho_{\pib}(N))$.
In this case, it suffices that $\rho_{\pib}(N) = o_p(N^{-1/4})$ and $\rho_f(N) = \mO_p(N^{-1/4})$.
We can also run LDR$^2$OPE again where $\InitialEstimate$ is outputted $\wh\alpha$ from the last LDR$^2$OPE run. Recursing $M$ times, the product rate becomes $\rho_{\pib}(N) \prns{ \rho_f(N) + \rho_{\pib}(N) \prns{ \rho_f(N) + \rho_{\pib}(N) (...) } } = \mO\prns{\rho_{\pib}(N) \rho_f(N) + \rho_{\pib}(N)^M \rho_\alpha(N)}$.
By iteratively refining localizations, we become more robust to a slower initial localization $\rho_\alpha(N)$, at the cost of more computation.

\cref{thm:dro_eval_ldml_efficiency} is significant even when behavior propensities are known, as LDR$^2$OPE improves over SNIPS in that LDR$^2$OPE is efficient and has a smaller asymptotic variance.
As remarked by \citet{kallus2019localized,kasy2019uniformity}, this theorem also holds uniformly over a family of nominal distributions $\PP_0$ under some regularity conditions, which implies a stronger finite-sample performance guarantee.

Finally, we note that while cross-fitting does require training regression models $K$ times, in practice this does not pose a computational burden, as $K = 2$ is sufficient for theory and in practice $K = 5$ is a reasonable choice. Furthermore, each cross-fitting run is identical, just running on different splits of the data, so they can be done in parallel. For a complete run-time analysis, please see \cref{sec:runtime-analysis-ldrope}.

\section{Doubly Robust DROPL}\label{sec:doubly_robust_dropl}

We now turn to distributionally robust off-policy learning, where we aim to find a policy with high distributionally robust value.
Ostensibly, DROPL involves DROPE for many policies, since to find a policy with high value, we need to be able to evaluate, or at least compare, different policies.
Directly applying the localization technique from LDR$^2$OPE does not help since an initial guess $\wh\alpha_{init}(\pi_1)$ for one policy may be far from $\alpha\opt(\pi_2)$ of another policy.
Thus, estimating a continuum of regression functions appears inevitable for the more challenging DROPL task.
This motivates us to directly apply doubly robust estimation to $W(\pi, \alpha)$, which requires estimating the continuum of regression functions $\braces{f_0(\cdot, \cdot; {\alpha}): \mS \times \mA \mapsto \R{}:  0 < \alpha \le \alphaub}$.

\subsection{Estimating a Continuum of Regression Functions}\label{sec: continuum-est}
We propose to estimate the continuum of regression functions $f_0(\cdot, \cdot; {\alpha})$ via a local weighting approach. Given $N$ data points, we first learn data-driven weighting functions $\braces{\est{\omega}_i(s,a)}_{i \in [N]}$ such that the conditional reward distribution $R \mid S=s, A = a$ can be approximated by $\sum_{i=1}^{N} \est{\omega}_i(s,a) \delta_{r_i}$, where $\delta_{r_i}$ is the Dirac measure at $r_i$.
Here, $\est{\omega}_i(s,a)$ roughly measures the proximity of the $i^{\op{th}}$ datapoint to the query point $(s, a)$, so it is typically larger when $(s_i, a_i)$ is closer to $(s, a)$.
Common weight construction methods include k-nearest neighbors, kernel regressions, decision trees and various tree ensembles \citep{bertsimas2020predictive,cevid2020distributional,khosravi2019non,oprescu2019orthogonal,meinshausen2006quantile,athey2019generalized}.
With these weights, we can approximate $f_0\prns{s, a; \alpha}$ for \emph{any} $\alpha$ with the following continuum estimator:
\begin{align}
	\wh f_0(s,a; \alpha) = \sum_{i=1}^N \est{\omega}_i(s,a) \exp(-r_i/\alpha). \label{eq:continuum_nuisances_weighted}
\end{align}
In our experiments, we constructed the weights using random forests \citep{breiman2001random}: we first run random forest to regress $R$ with respect to $(S, A)$, and then compute $\est{\omega}_i(s,a)$ as the average frequency that data point $(s_i, a_i)$ and query point $(s, a)$ lie in the same tree leave node. This method has been successfully applied in statistical estimation and decision making \citep[\textit{e.g., }][]{bertsimas2020predictive,meinshausen2006quantile,kallus2020stochastic}.

\subsection{Learning Algorithm}
In \cref{algo:dropl_cdr}, we propose \textbf{C}ontinuum \textbf{D}oubly \textbf{R}obust \textbf{DROPL} (CDR$^2$OPL), which targets the policy $\wh\pi^{DR}$ that maximizes the doubly robust objective. It does so by jointly optimizing the dual variable $\alpha$ and policy (\textit{e.g.}, by policy gradient updates) in an alternating fashion. We fit the continuum of regressions in \cref{line:fit-continuum}.
\begin{align}
	&\wh\pi^{DR} \in \argmax_{\pi \in \Pi} \wh\droValue^{DR}(\pi) \label{eq:dro_learning_erm}
	\\&\wh\droValue^{DR}(\pi) \defeq \max_{\alpha > 0} -\alpha \log \wh W^{DR}(\pi,\alpha) - \alpha\delta \nonumber
    \\&\wh W^{DR}(\pi,\alpha) \defeq \frac{1}{N} \sum_{k=1}^K \sum_{i \in \mI_k} \frac{\pi(a_i\mid s_i)}{\wh\pib^{(k)}(a_i\mid s_i)} \Big( \exp(-r_i/\alpha) \nonumber
    \\&\phantom{\wh W^{DR}(\pi,} - \wh f_0^{(k)}(s_i,a_i; \alpha) \Big) + \sum_{a\in\mA}\pi(a\mid s_i){\wh f_0^{(k)}(s_i,a; \alpha)}. \nonumber
\end{align}

\begin{algorithm}[t!]
	\caption{Continuum Doubly Robust DROPL}
	\label{algo:dropl_cdr}
	\begin{algorithmic}[1]
		\State \textbf{Input:} Data $\mD$, policy class $\Pi$, uncertainty set radius $\delta$.
		\State Randomly split $\mD$ into $K$ (approximately) even folds, with the indices of the $k^{\text{th}}$ fold denoted as $\mI_k$.
		\For{$k=1,...,K$}
		    \State Using $\mD[\mI_k^C]$,
		    train $\wh\pib^{(k)}$ to fit $\pib$.
		    \State \multiline{Using $\mD[\mI_k^C]$,
		    train $\wh f_0^{(k)}(\cdot; {\alpha})$ to fit $f_0(\cdot; {\alpha})$ for all $\alpha \in (0, \alphaub)$, e.g. using \cref{sec: continuum-est}. \label{line:fit-continuum}}
		\EndFor
		\State Initialize $\wh\pi$.
		\While{$\wh\pi$ has not converged}
		    \State Set $\wh\alpha \gets \argmax_{\alpha > 0} -\alpha \log \wh W^{DR}(\wh\pi,\alpha) - \alpha \delta$. \label{line:dro_learning_update_alpha}  %
		    \State \multiline{Update the policy $\wh\pi$ (\textit{e.g.}, take some gradient steps) to minimize $\wh W^{DR}(\pi,\wh\alpha)$. \label{line:dro_learning_update_pi}}
		\EndWhile
		\State \textbf{Return: } $\wh\pi$.
	\end{algorithmic}
\end{algorithm}

\subsection{Regret Bounds}\label{sec:learning_guarantees}
We now derive a finite-sample distributionally robust regret guarantee for $\wh\pi^{DR}$ (from \cref{eq:dro_learning_erm}).
We adopt the Hamming entropy integral $\kappa(\Pi)$ from \citet{si2020distributional} as a complexity measure for the policy class $\Pi$.
Recall the Hamming distance between two policies is the fraction of mismatched action distributions in the dataset,
\begin{align*}
\distHamming{\pi_1}{\pi_2} = \frac{1}{N} \sum_{i=1}^N \I{\pi_1(s_i) \neq \pi_2(s_i)}.
\end{align*}
Then, the Hamming covering number $\mC({\eps,\Pi; \braces{s_i}_{i \in [N]}})$ is the cardinality of the smallest set of policies $\wt\Pi$ such that for any $\pi \in \Pi$, there exists $\wt\pi \in \wt\Pi$ with $\distHamming{\pi}{\wt\pi} \leq \eps$.
Denote the largest size over all datasets as $\mN\prns{\eps,\Pi} \defeq \sup_{N \geq 1} \sup_{\braces{s_i}_{i \in [N]}} \abs{ \mC\prns{\eps,\Pi; \braces{s_i}_{i \in [N]}} }$.

\begin{definition}
The Hamming entropy integral of $\Pi$ is
\begin{align*}
    \kappa(\Pi) \defeq \int_0^1 \log^{1/2} \mN\prns{t^2,\Pi} \diff t.
\end{align*}
\end{definition}
For example, if $\Pi$ is finite, we have $\kappa(\Pi) \leq \log^{1/2}(|\Pi|)$.

Since CDR$^2$OPL fits a continuum of regressions, our guarantee involves the uniform estimation rate over the continuum. 
\begin{definition}
Suppose $\{\wh f_0^{(k)}(\cdot, \alpha), \alpha\in[\alphalb,\alphaub]\}$ is learned from a dataset of $\frac{N(K-1)}{K}$ points.
For any $\beta\in(0,1)$, define $\op{Rate}_f^{\mathfrak{c}}(N,\beta)$ so that w.p. at least $1-\beta$, it upper bounds $\nmP{ \supoveralpha \abs{ \wh f_0^{(k)}(S,A;\alpha) - f_0(S,A;\alpha) } }$.
\end{definition}
Similarly, let $\op{Rate}_{\pib}(N,\beta)$ be the estimation rate for $\wh\pib^{(k)}$.
Unlike the rates we used for the asymptotic theory of LDR$^2$OPE, these rates are deterministic functions of $N$ and $\beta$, which is needed for our finite-sample guarantee.
We now state our main guarantee for DROPL.
\begin{restatable}{theorem}{regretbound}\label{thm:cts_reward_dr_regret}
Suppose \cref{asm:standard_cb,asm:alpha_bounded_from_zero}.
Then, for any $\beta \in (0, 1/6)$, w.p. at least $1-6\beta$, the distributionally robust regret $\droRegret{\wh\pi^{DR}}$ is at most
\begin{align*}
&\frac{2112 \alphaub \sqrt{K}}{\wlb \piblb \sqrt{N}} \prns{\kappa(\Pi) + \frac{\alphaub}{\alphalb^2} + \log^{1/2}(K/\beta)}
\\& + \frac{4\alphaub}{\wlb\piblb^2} \prns{\op{Rate}_{\pib}(N,\beta/K) \cdot \op{Rate}_f^{\mathfrak{c}}(N, \beta/K)},
\end{align*}
provided $N$ is sufficiently large (\cref{asm:sufficiently_large_n}).
\end{restatable}
Please see \cref{sec:learning_guarantees_proofs} for the proof.
This theorem shows that $\wh\pi^{DR}$ achieves distributionally robust regret with a $\mO(N^{-1/2})$ term, plus a product rates term $\mO(\op{Rate}_{\pib}(N,\beta) \cdot \op{Rate}_f^{\mathfrak{c}}(N,\beta))$.
We highlight that our dependence on nuisance estimation is manifested as the product of rates, which allows for non-parametric (sub-$\sqrt{N}$) rates for each nuisance.
For example, if the estimation rates $\op{Rate}_{\pib}(N,\beta)$ and $\op{Rate}_f^{\mathfrak{c}}(N,\beta)$ are both $o(N^{-1/4})$, the contribution of this product term is lower order, and the regret is $\mO(N^{-1/2})$.
If the estimated propensities $\wh\pib^{(k)}$ are obtained by empirical risk minimization methods, then the results in \citet{wainwright2019high,bartlett2005local} can be used to show that $\op{Rate}_{\pib}(N,\beta) \leq C(\frac{1}{N^p} + \sqrt{\log(1/\beta)/N})$ where the rate $p$ depends on the complexity of the function class, such as given by its metric entropy.
The rate of convergence for the continuum nuisance $\op{Rate}_f^{\mathfrak{c}}(N,\beta)$ can be argued based on analysis of \citet{bertsimas2020predictive,belloni2017program}.
In the proof, we use \cref{asm:alpha_bounded_from_zero} to show that $\alpha \mapsto \exp(-r/\alpha)$ (and the expectation variants) is Lipschitz, with Lipschitz constant at most $1/\alphalb^2$ (\cref{lm:covering_number_for_w}). This implies that the continuum estimator proposed in \cref{sec: continuum-est}, as a convex combination of Lipschitz functions, is also Lipschitz. We remark that point-wise rates provided in \citet{cevid2020distributional,oprescu2019orthogonal,athey2019generalized,gyorfi2002distribution} can then be translated into uniform rates, thanks to this Lipschitz property (see \cref{lm:pointwise-to-lipschitz}).

\begin{figure*}[!h]
	\centering
	\includegraphics[width=\linewidth]{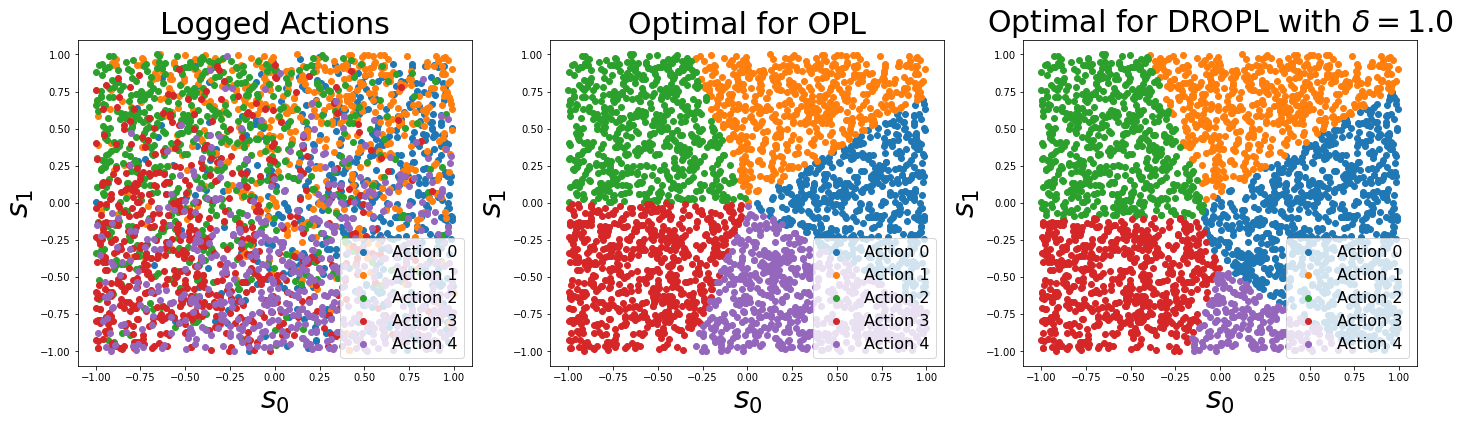}
	\caption{Scatter plots of (Left) the behavior policy, (Center) the optimal policy for expected reward, and (Right) the optimal policy for distributionally robust value with $\delta = 1.0$. 
	Notice that in the right-most plot, the distributionally robust policy prefers \textcolor{blue}{Action 0} more, since its conditional reward has the lowest variance, and so choosing it is more robust. }
	\label{fig:visualize-experiment}
\end{figure*}

\section{Experiments} \label{sec:decision_tree_policy_experiments}
We evaluated our doubly robust algorithms for DROPE/L in a simulated setting where distributional shifts can be easily visualized. The following is our data generating process $\PP_0$.
The state space is two-dimensional $\mS = [-1, 1]^2$, and states are sampled uniformly $S \sim \text{Unif}([-1,1]^2)$.
The action space is $\mA = \braces{0, 1, \dots, 4}$, and the behavior policy is a softmax policy $\pib(a \mid s) \propto \exp(2s\tr \beta_a)$, where $\beta_a$'s are the coordinates of the $k$-th fifth root of unity, i.e. $\beta_a = (\operatorname{Re}{\zeta_a}, \operatorname{Im}{\zeta_a})$ where $\zeta_a = \exp(2a\pi i/5)$. Potential outcomes are normally distributed: $R(a) \mid S=s \sim \mN(s\tr \beta_a, \sigma_a^2)$, where $\sigma = [0.1, 0.2, 0.3, 0.4, 0.5]\tr$. This setup is visualized in \cref{fig:visualize-experiment}. We see that the optimal policy for OPL partitions the state space in equal angles, based on which root of unity the given state is closest to, while the optimal policy for DROPL favors actions with lower variances in the reward. This connection of KL-DRO to variance regularization has been studied in the DRO literature \citep{lam2016robust,DBLP:journals/jmlr/DuchiN19}.

First, we compared our DROPE proposal, LDR$^2$OPE (\cref{algo:dro_eval_ldml}), to the SNIPS-based evaluation baseline \citep[Algorithm 1]{si2020distributional}. The target policy we seek to evaluate is $\pi_{target}(a \mid s) \propto \exp(s\tr \beta_a)$, which is like the behavior policy $\pi_0$ but with a different softmax temperature.
We conducted experiments under three uncertainty set radii $\delta = 0.1, 0.2, 0.3$, and in two settings, where propensities $\pib$ were known and unknown.
If propensities were known, both LDR$^2$OPE and SNIPS used ground truth propensities $\pib$.
If propensities were unknown, both methods used estimated propensities obtained from Gradient Boosted Trees using the LightGBM package \citep{ke2017lightgbm}. We also used LightGBM for regressing LDR$^2$OPE's outcome functions $\wh f_j^{(k)}$ for $j=0,1$. 
We self-normalized the propensity weights for our proposed doubly robust methods as we found it beneficial in the small $N$ regime.
All models were fitted with $K=5$ fold cross-fitting, and we repeated this over $30$ seeds. Shaded regions in plots are $90\%$ confidence intervals computed with the bootstrap in Seaborn \citep{waskom2021seaborn}.

\begin{figure*}[!h]
	\centering
	\includegraphics[width=0.95\linewidth]{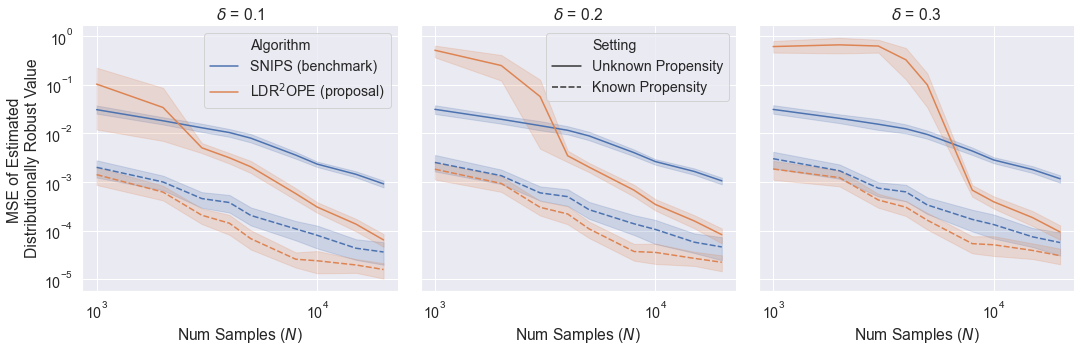}
	\caption{Comparison of our proposal \textcolor{orange}{LDR$^2$OPE} to the baseline \textcolor{blue}{SNIPS} in the DROPE task, repeated for $\delta = 0.1,0.2,0.3$.
	Solid and dashed lines denote settings when behavior propensities are unknown and known, respectively.
	The x-axis is the number of samples $N$ used by the evaluation algorithm, and the y-axis is the mean squared error (MSE) of the DROPE estimator, so lower is better.
	When $N$ is large enough, we see that LDR$^2$OPE has lower MSE than SNIPS in all cases.
	}
	\label{fig:eval}
\end{figure*}

\cref{fig:eval} shows the results of the DROPE experiments.
In all experimental setups, as long as $N$ is large enough, we observe that LDR$^2$OPE outperforms SNIPS. Importantly, LDR$^2$OPE has a faster rate of MSE decrease. 
However, in the setting when $N$ is small (non-asymptotic regime), propensities are unknown, and $\delta$ is large, LDR$^2$OPE may be less stable than SNIPS; that is, doubly robust appears to suffer when all three challenges arise, but as long as one challenge is mitigated, doubly robust offers a significant improvement over baseline.
While performance of both methods deteriorated, as expected, when $\pi_0$ was not known had to be estimated, we see that whenever $N \geq 10^4$, LDR$^2$OPE with estimated propensities is actually competitive with the algorithms with access to the ground truth $\pi_0$ \emph{a priori}.
This empirically reinforces our theory that LDR$^2$OPE is asymptotically optimal, even when propensities are estimated with flexible, non-parametric ML methods. Overall, except in the setting with small $N$, unknown propensities and large $\delta$, LDR$^2$OPE offers a significant benefit over SNIPS.

Next, we compared our DROPL proposal, CDR$^2$OPL (\cref{algo:dropl_cdr}), to maximizing the SNIPS objective \citep[Algorithm 2]{si2020distributional}.
In CDR$^2$OPL, the continuum of regression functions $\{\wh f_0(s,a);\alpha\}$ was estimated according to \cref{sec: continuum-est}, with weights $\wh\omega_i(s,a)$ derived from fitting a Random Forest with $25$ trees.
Our policies were neural network softmax policies with a hidden layer of $32$ neurons and ReLU activation. For Line~\ref{line:dro_learning_update_pi}, we minimized $\wh W^{DR}(\cdot,\alpha)$ using Adam with a learning rate of $0.01$. Following \citet{dudik2011doubly}, we repeated each policy update ten times with perturbed starting weights and picked the best weights based on training objective, since the doubly robust estimate $\wh W^{DR}(\cdot,\alpha)$ is non-convex in the policy weights. %

\cref{fig:learning} shows the results of the DROPL experiments. 
When $\delta=0.1$, CDR$^2$OPL consistently learns policies that improve over the baseline distributionally robust value by about $1\%$, but this benefit from double robustness becomes less significant as $\delta$ grows, a trend we also saw in the DROPE experiments.
Here, this decrease in improvement may be due to the fact that as $\delta$ increases to infinity, the distributionally robust value of all policies converge to the minimum reward, and so the policy improvement becomes less noticeable.
We also highlight that, while CDR$^2$OPL offers a performance improvement at least for smaller $\delta$, it comes at a computational cost. This is because each call to the estimator $\wh f_0(s,a;\alpha)$ requires a weighted sum over the training dataset, rendering the overall running time for CDR$^2$OPL to be $\mO(N^2)$, while it is $\mO(N)$ for SNIPS. Further, the necessity of restarting policy optimization at many random starting weights to combat non-convexity of $\wh W^{DR}(\cdot,\alpha)$ also increases computational cost by a constant factor. Given the computational and optimization challenges of CDR$^2$OPL, resulting in the potentially marginal improvement for large $\delta$, this investigation reveals that (cross-fitted) SNIPS still remains an attractive choice in many practical situations.
Our recommendation is to try both learning algorithms, and then select the better one by evaluating with LDR$^2$OPE.
Finding a more computationally efficient and stable algorithm for DROPL with unknown propensities is an interesting direction for future work.

\begin{figure*}[!h]
	\centering
	\includegraphics[width=\linewidth]{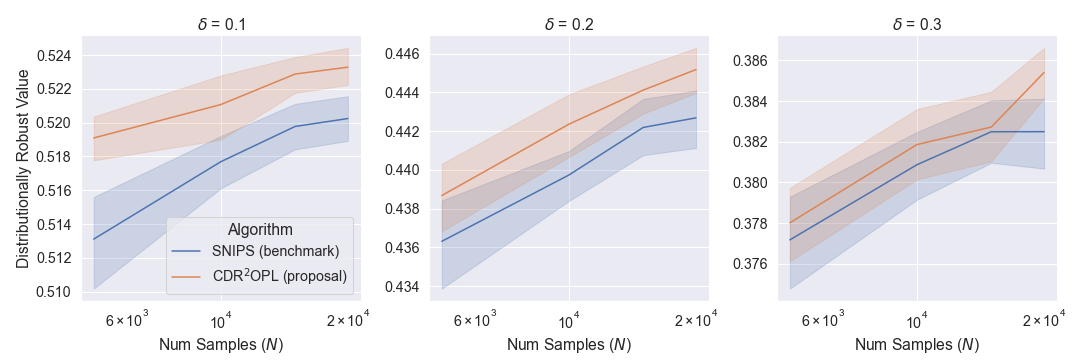}
	\caption{Comparison of our proposal \textcolor{orange}{CDR$^2$OPL} to the baseline \textcolor{blue}{SNIPS} in the DROPL task, repeated for $\delta=0.1,0.2,0.3$.
	The x-axis is the number of samples $N$ used by the algorithm. The y-axis is the distributionally robust value $\droValue$ of the learned policy, so higher is better.}
	\label{fig:learning}
\end{figure*}

\section{Extension to $f$-divergences}\label{sec: f-divergence}
Now, we generalize our results to uncertainty sets generated by the $f$-divergence $D_f$. Recall that for any convex function $f: \Rl^+ \to \Rl$ satisfying $f(1) = 0$, the $f$-divergence is defined as $D_f(P \Mid Q) \defeq \E_Q\bracks{f\prns{\diff P/\diff Q}}$  \citep{sason2016f}, and recall that $f^*(z) \defeq \sup_{x > 0} \ip{z}{x} - f(x)$ is the Fenchel conjugate of $f$.
Strong duality gives a variational form for the distributionally robust value, now with a second dual variable $\lambda$ \citep{namkoong2016stochastic},
\begin{align}
	&\droValue^f(\pi) = \sup_{\alpha \geq 0, \lambda \in \Rl} \phi^f(\pi, \alpha, \lambda) \label{eq:def_phi_f_div} \\
	&\phi^f(\pi, \alpha, \lambda) = -\alpha \E_{\PP_0}\bracks{f^*\prns{\frac{-R(\pi(S)) - \lambda}{\alpha}}} - \alpha \delta - \lambda.  \nonumber
\end{align}
Akin to the KL-DRO setting, \cref{eq:def_phi_f_div} has a unique solution $(\alpha^\star, \lambda^\star)$, which is the root to $\nabla \phi^f = \tb{0}$.
To extend LDR$^2$OPE to $f$-divergence DROPE, we can solve the doubly robust moment equation
in \cref{eq:ortho_moment_drope} with $\theta = [ \alpha, \lambda, W_0^f, W_1^f, W_2^f, \droValue^f ]^\top$ and $U, V$ as
\resizebox{1.05\linewidth}{!}{
\(
U(r; \alpha, \lambda) = 
\begin{bmatrix}
	f^*\prns{ \frac{-r-\lambda}{\alpha} } \\
	(f^*)'\prns{ \frac{-r-\lambda}{\alpha} } \\
	(f^*)'\prns{ \frac{-r-\lambda}{\alpha} } \cdot (r + \lambda) \\
	0 \\
	0 \\
	0
\end{bmatrix},
V(\theta) = 
\begin{bmatrix}
	-W_0^f \\
	-W_1^f \\
	-W_2^f \\
	W_1^f - 1 \\
	-W_0^f - W_2^f/\alpha - \delta \\
	-\droValue -\alpha W_0^f - \alpha \delta - \lambda
\end{bmatrix}.
\)
}

By a similar argument to \cref{thm:dro_eval_ldml_efficiency}, the resulting localized doubly robust estimator is asymptotically linear and enjoys semiparametric efficiency.
Note that we could have solved the KL problem by setting $f(x) = x\log(x)$ in \cref{eq:def_phi_f_div}, and solving a supremum over $\alpha$ and $\lambda$ jointly.
This would also be efficient and thus have the same asymptotic variance as \cref{thm:dro_eval_ldml_efficiency}. 
But since the supremum over $\lambda$ can be solved in a closed-form way that recovers \cref{eq:def_phi} (see \cref{sec:kl_from_f_div}), our direct analysis for KL should yield better empirical results since we don't need to optimize over $\lambda$.
In \cref{sec:cressie_read_eval}, we also discuss a direct analysis of the Cressie-Read divergences, which has a closed form solution for the supremum over $\alpha$.

\section{Concluding Remarks}
In this paper, we present LDR$^2$OPE and CDR$^2$OPL, the first doubly robust methods for distributionally robust off-policy evaluation (DROPE) and learning (DROPL), respectively. By virtue of being both distributionally robust and doubly robust, our methods are robust to environment shifts and slow nuisance estimations.
By leveraging a localization technique, LDR$^2$OPE only needs to fit two outcome functions, instead of a continuum of outcome functions. We prove that LDR$^2$OPE is $\sqrt{N}$-consistent, asymptotically linear and enjoys semiparametric efficiency for DROPE, and empirical showed that it offers significant benefits over the SNIPS baseline.
Our learning method CDR$^2$OPL fits a continuum of outcome functions using a data-driven weighting approach. Under a product rate condition, we prove that CDR$^2$OPL achieves $\mO(N^{-1/2})$ regret, via a uniform coupling over the dual variables that generalizes previous results \citet{zhou2018offline}. Given additional computational overhead of CDR$^2$OPL, the simplicity and stability of SNIPS renders it still quite attractive. Our suggestion for practitioners is to try several DROPL methods, e.g. CDR$^2$OPL and cross-fitted SNIPS, and select the best one using LDR$^2$OPE. Developing a more computationally efficient and stable algorithm for DROPL with non-parametrically estimated propensities is an interesting direction for future work. Another promising next step is to develop methods to deal with unknown Wasserstein environment shifts, which could be a more intuitive metric when contexts are images. Finally, we are interested in generalizing our techniques to distributionally robust reinforcement learning, for which an emerging line of work~\citet{zhou2021finite,panaganti2022sample,smirnova2019distributionally,DistributionalRobustQ-learning} has been devoted to studying the simulator access model, which is hence no harder than the known data collection setting. Generalizing our techniques here to the RL setting would be worthwhile and challenging, which we leave for future work.

\subsubsection*{Acknowledgements}
This material is based upon work supported by the National
Science Foundation under Grants No. IIS-1939704 and CCF-2106508, a JP Morgan research grant, and a Cornell University Fellowship. We thank Ban Kawas, Nian Si, and the anonymous reviewers for useful discussions and feedback.

\newpage
\bibliographystyle{icml2022}
\bibliography{main}

\appendix
\onecolumn

\setlength{\parindent}{0pt}
\setlength{\parskip}{\baselineskip}

\begin{center}\LARGE
\textbf{Appendices}
\end{center}

\section{DRO Calculations} \label{sec:useful_calculations}
In this section, we list some useful calculations for distributionally robust optimization (DRO).
\subsection{KL-divergence DRO} \label{sec:kl_calculations}
Recall from \cref{eq:def_W,eq:def_phi}
\begin{align*}
	W(\pi, \alpha) &:= \Eb{\exp(-\piR/\alpha)} \\
	\phi(\pi, \alpha) &:= -\alpha \log W(\pi, \alpha) - \alpha \delta
\end{align*}

1-st derivatives w.r.t. $\alpha$:
\begin{align*}
	\partialwrt{\alpha} W(\pi, \alpha) 
	&= \frac{1}{\alpha^2} \Eb{\piR \exp(-\piR/\alpha)} \\
	\partialwrt{\alpha} \phi(\pi, \alpha) 
	&= -\log W(\pi, \alpha) - \alpha \frac{\partialwrt{\alpha} W(\pi, \alpha)}{W(\pi, \alpha)} - \delta \\
	&= -\log W(\pi, \alpha) - \frac{\Eb{\piR \exp(-\piR/\alpha)}}{\alpha \cdot \Eb{\exp(-\piR/\alpha)}} - \delta
\end{align*}

2-nd derivatives w.r.t. $\alpha$:
\begin{align*}
	\partialwrtmulti{\alpha}{2} W(\pi, \alpha) &= -\frac{2}{\alpha^3} \Eb{\piR \exp(-\piR/\alpha)} \\
	&+ \frac{1}{\alpha^4} \Eb{\piR^2 \exp(-\piR/\alpha)} \\
	\partialwrtmulti{\alpha}{2} \phi(\pi, \alpha) 
	&= -2\frac{\partialwrt{\alpha} W(\pi, \alpha)}{W(\pi, \alpha)} - \frac{\alpha}{W(\pi, \alpha)^2} \prns{ \prns{\partialwrtmulti{\alpha}{2}W(\pi, \alpha)} \cdot W(\pi, \alpha) - \prns{\partialwrt{\alpha}W(\pi, \alpha)}^2 } \\
	&= \frac{1}{\alpha^3 \Eb{\exp(-\piR/\alpha)}} \prns{ \frac{\prns{\Eb{\piR \exp(-\piR/\alpha)}}^2}{\Eb{\exp(-\piR/\alpha)}} - \Eb{\piR^2 \exp(-\piR/\alpha)}}
\end{align*}

1-st derivatives w.r.t. $\pi$'s parameters $\theta$:
\begin{align*}
	\partialwrt{\pi} \phi(\pi, \alpha) = -\frac{\alpha}{W(\pi, \alpha)} \Eb{\exp(-\piR/\alpha) \cdot \nabla_\theta \log \pi(A|S)}
\end{align*}

\subsection{$f$-divergence DRO} \label{sec:fdiv_calculations}
Recall 
\begin{align*}
	 \phi^f(\pi, \alpha, \lambda) &= -\alpha \Eb{f^*\prns{\frac{-\piR - \lambda}{\alpha}}} - \alpha \delta - \lambda
\end{align*}
Then,
\begin{align*}
	\partialwrt{\lambda} \phi(\alpha, \lambda) &= -\alpha \Eb{ (f^{*})' \prns{ \frac{-\piR - \lambda}{\alpha} } \cdot \prns{ \frac{-1}{\alpha} } } - 1 \\
	&= \Eb{ (f^{*})'\prns{ \frac{-\piR - \lambda}{\alpha} } } - 1 \\
	\partialwrt{\alpha} \phi(\alpha, \lambda) 
	&= -W(\alpha, \lambda) - \alpha \Eb{ (f^{*})'\prns{ \frac{-\piR - \lambda}{\alpha} } \cdot \prns{ \frac{\piR + \lambda}{\alpha^2} } } - \delta \\
	&= -W(\alpha, \lambda) - \frac{1}{\alpha}\Eb{ (f^{*})'\prns{ \frac{-\piR - \lambda}{\alpha} } \cdot \prns{\piR + \lambda} } - \delta
\end{align*}

\subsection{Recovering KL duality} \label{sec:kl_from_f_div}
Recall that the KL divergence is an $f$-divergence, where $f_{KL}(x) = x\log(x)$ with the dual $f_{KL}^*(x) = \exp(x-1)$.
\begin{align*}
	\phi^{KL}(\pi, \alpha, \lambda) &= -\alpha \exp(-(\lambda + 1)/ \alpha) \E_{\P_0} \bracks{\exp\prns{-\piR/\alpha} } - \alpha \delta - \lambda \\
	\partialwrt{\lambda}\phi^{KL}(\pi,\alpha,\lambda) &= \exp(-(\lambda + 1)/ \alpha) \E_{\P_0} \bracks{\exp\prns{-\piR/\alpha}} - 1
\end{align*}
Setting this to 0, yields:
\begin{align*}
	\exp((\lambda^* + 1)/ \alpha) &= \E_{\P_0} \bracks{\exp\prns{-\piR/\alpha} } \\
	\lambda^* &= \alpha \log \prns{ \Eb{\exp(-\piR/\alpha} } - \alpha
\end{align*}
Plugging this into the original expression, we get the same equation as \cref{eq:def_phi},
\begin{align*}
	&-\alpha \exp(-(\lambda^* + 1)/ \alpha) \Eb{\exp\prns{-\piR/\alpha} } - \alpha \delta - \lambda^* \\
	&=-\alpha - \alpha \delta - \alpha \log \prns{ \Eb{\exp(-\piR/\alpha} } + \alpha \\
	&= \phi^{KL}(\pi, \alpha)
\end{align*}

\subsection{KL DRO Lemmas}
In this section, we prove useful lemmas about the KL DROPE objective $\phi$ (\cref{eq:def_phi}).

First, we show that $\phi(\pi, \alpha)$ is strictly concave in $\alpha$, except when $\piR$ is almost surely a constant. This corner case implies that $\Eb[\PP_0]{f(\piR)} = f(\piR)$ for any measurable function $f$, and hence the distributionally robust objective simplifies to the constant $\piR$, since $\sup_{\alpha > 0} \phi(\pi,\alpha) = \sup_{\alpha > 0} \piR - \alpha\delta = \piR$.

\begin{lemma}[$\phi$ is strictly concave]\label{lm:phi_concavity}
$\forall \alpha > 0: \partialwrtmulti{\alpha}{2} \phi(\pi, \alpha) \leq 0$, with strict inequality iff $\piR$ is not almost surely a constant.
\end{lemma}
\begin{proof}
	By Cauchy-Schwartz in $L_2$,
	\begin{align*}
		\Eb{\piR \exp(-\piR/\alpha)} \leq \sqrt{\Eb{\exp(-\piR/\alpha)} \Eb{\piR^2 \exp(-\piR/\alpha)}}
	\end{align*}
	with equality iff $\piR$ and $\exp(-\piR/\alpha)$ are colinear, which happens iff $\piR$ is almost surely constant.
	By calculations in \cref{sec:useful_calculations}, we have
	\begin{align*}
		\partialwrtmulti{\alpha}{2} \phi(\pi, \alpha) 
		= \frac{1}{\alpha^3 \Eb{\exp(-\piR/\alpha)}} \prns{\frac{ \prns{\Eb{\piR \exp(-\piR/\alpha)}}^2}{\Eb{\exp(-\piR/\alpha)}} -\Eb{\piR^2 \exp(-\piR/\alpha)}} < 0
	\end{align*}
	Thus, $\phi(\pi, \alpha)$ is concave in $\alpha$, and strictly concave unless $\piR$ is almost surely a constant.
\end{proof}

Next, we show that under the reward coverage assumption (in \cref{asm:standard_cb}), we can lower bound $W$ (\cref{eq:def_W}).
\begin{lemma}[Lower bound of $W$]\label{lm:w_lower_bound}
	Under \cref{asm:standard_cb}, we can lower bound $W(\pi,\alpha)$ as follows:
	\begin{enumerate}[label=(\roman*)]
		\item If $R(a) \mid S$ is continuous, $W(\pi,\alpha) \geq \frac{\rewardSupportLb}{2} \min(\alpha, 1)$.
		\item If $R(a) \mid S$ is discrete, $W(\pi,\alpha) \geq \rewardSupportLb$.
	\end{enumerate}
\end{lemma}
\begin{proof}[Proof of \cref{lm:w_lower_bound}]
    Let $\mu_\pi$ represent the distribution of over $\mS \times \mA$, of $s \sim \PP_0, a \sim \pi(s)$.
	\\\noindent \tb{Proof of discrete case} \\
	First, it's easier to see the discrete case,
	\begin{align*}
		W(\pi,\alpha) 
		= \int_{\mS \times \mA} \sum_{r \in \mathbb{D}} p_R(r\mid s,a) \exp(-r/\alpha)  \diff\mu_\pi
		\geq \int_{\mS \times \mA} \sum_{r \in \mathbb{D}} \rewardSupportLb \exp(-0/\alpha)  \diff\mu_\pi
		= \rewardSupportLb
	\end{align*}
	since $0 \in \mathbb{D}$ by \cref{asm:standard_cb}.
	\\ \noindent \tb{Proof of continuous case} \\
	In the continuous case,
	\begin{align*}
		W(\pi,\alpha) 
		= \int_{\mS \times \mA} \int_0^1 f(r) \exp(-r/\alpha) \diff r  \diff\mu_\pi
		\geq \rewardSupportLb \alpha \int_{S \times A} \int_0^{1/\alpha} \exp(-r) \diff r  \diff\mu_\pi
		= \rewardSupportLb \alpha \prns{1 - \exp(-1/\alpha)}
	\end{align*}
	
	To remove the exponentiation, observe that $\alpha(1-\exp(-1/\alpha))$ is increasing and concave, so we can lower-bound the first part by a line with an appropriately chosen slope, and the second part by the intersection point with the line.
	For some slope $m$ which we'll set later, the two points of intersection between $\alpha(1-\exp(-1/\alpha))$ and the line $\alpha m$ are $(x_1, y_1) = (0, 0)$ and $(x_2, y_2) = (\frac{1}{-\log(1-m)}, \frac{1}{-\log(1-m)} m)$, which can be seen by solving the following:
	\begin{align*}
		\alpha (1-\exp(-1/\alpha)) = \alpha m
	\end{align*}
	From $\alpha \in [0, x_2]$, we have $\alpha (1-\exp(-1/\alpha)) \geq \alpha m$, and for $\alpha \geq x_2$, we have $\alpha (1-\exp(-1/\alpha)) \geq y_2$. Hence, we have
	\begin{align*}
		\alpha(1-\exp(-1/\alpha)) \geq \min\prns{ \alpha m, \frac{m}{-\log(1-m)} 1 }
	\end{align*}
	
	We can choose $m_c$ so that $c = \frac{m_c}{-\log(1-m_c)}$ for some chosen constant $0 < c < 1$ (sufficient and necessary to be less than 1, since $\sup_{\alpha > 0} \alpha (1-\exp(-1/\alpha)) = 1$). For example setting $c = 0.5$ gives $m_c \approx 0.797$ implies 
	\begin{align*}
		\alpha(1-\exp(-1/\alpha)) \geq \min( 0.797 \alpha, 0.5 1 ) \geq \frac{\min(\alpha, 1)}{2}.
	\end{align*}
	Thus, $W(\pi,\alpha) \geq \rewardSupportLb \alpha(1-\exp(-1/\alpha)) \geq \rewardSupportLb \frac{\min(\alpha, 1)}{2}$.
\end{proof}

\subsection{Cressie-Read divergence DRO} \label{sec:cressie_read_eval}
For $k > 1$, the $k$-Cressie-Read divergence is the $f$-divergence where 
\(
f_k(t) = \frac{1}{k} - \frac{t}{k-1} + \frac{1}{k-1} \frac{t^k}{k}
\)
\citep{cressie1984multinomial}. 
While KL had a close form solution for $\lambda^*$ in \cref{eq:def_phi_f_div}, Cressie-Read divergences have a close form solution for $\alpha^*$, as shown in the Appendix of \citep{duchi2018learning}. Shown in \cref{eq:def_phi_cressie_read}, the dual expression for Cressie-Read divergences is a supremum over just $\lambda$.
\begin{align}
	\droValue^k(\pi) &= \sup_{\lambda \in \Rl} \phi^k(\pi, \lambda) \label{eq:def_phi_cressie_read} \\
	\text{where } \phi^k(\pi, \lambda) &= -c_k(\delta) \Eb{ (-\piR - \lambda)_+^{k_*} }^{1/k_*} - \lambda  \nonumber
\end{align}
By concavity, $\lambda^*$ is the unique solution to $\nabla_\lambda \phi^k = \tb{0}$. Thus, the Cressie-Read LDR$^2$OPE is to run \cref{algo:dro_eval_ldml}, with $\theta, U, V$ defined by
\begin{align}
	&\theta = \braces{\lambda, W_0, W_1, Q}, \label{eq:def_drope_uv_cressie_read} \\
	&U(r; \lambda) = 
	\begin{bmatrix}
		(-r - \lambda)^{k_*}_+ \\
		(-r - \lambda)^{k_*-1}_+ \\
		0 \\
		0
	\end{bmatrix} 
	\qquad V(\theta) = 
	\begin{bmatrix}
		-W_0 \\
		-W_1 \\
		-c_k(\delta) W_0^{1/k_*-1} \cdot W_1 + 1 \\
		-Q - c_k(\delta) W_0^{1/k_*} - \lambda \\
	\end{bmatrix}.  \nonumber
\end{align}
The algorithm is asymptotically linear and enjoys semiparametric efficiency.

\section{Degeneracy of Weighted DROPE Estimators}\label{sec:weighted_estimators_degenerate_drope}
While not crucial to the development of our DR estimators, we now digress to describe and characterize a blow-up phenomenon arising from the non-linear and supremum structure of the DROPE objective. Prior work found self-normalized IPS for DROPE to be empirically more stable than IPS \citep{si2020distributional}. 
When the propensity ratios were small, we actually found IPS to explode and have infinite estimation error! 
From the point of view of OPE, this is surprising since the difference between IPS and SNIPS would never be as extreme as infinite. \cref{thm:ips_solution_characterization} theoretically characterizes when this explosion occurs for any weighted estimator for $W$. For non-negative weights $\braces{w_i, i \in [N]}$, define weight-mean
\(
\ipsMean := \frac{1}{N} \sum_{i=1}^N w_i 
\)
and min-reward weight-mean
\(
\minRewardIpsMean := \frac{1}{N} \sum_{r_i = m} w_i
\), where $m = \min_i r_i$.

The weighted estimator we consider is
\begin{align}
	&\est{\phi}(\pi,\alpha) = -\alpha \log \prns{ \frac{1}{N} \sum_{i=1}^N w_i \exp(-r_i/\alpha) } - \alpha \delta \nonumber \\
	&\est{\droValue}(\pi) = \sup_{\alpha > 0} \est{\phi}(\pi, \alpha) \label{eq:drope_weighted_est}
\end{align}
If $w_i = \pi(a_i\mid s_i)/\pib(a_i\mid s_i)$, then this is IPS. If $w_i = \frac{\pi(a_i\mid s_i)/\pib(a_i\mid s_i)}{\frac{1}{N} \sum_{i=1}^N \pi(a_i\mid s_i)/\pib(a_i\mid s_i)}$, then this is SNIPS. Observe that for SNIPS, we have the mean of the weights is $S_w = 1$. This property turns out to be important in the characterization below.

\begin{theorem} \label{thm:ips_solution_characterization}
	Let $\delta > 0$, and let $\est{\alpha}$ be the empirical solution to \cref{eq:drope_weighted_est}. 
	Then, under \cref{asm:standard_cb},
	\begin{enumerate}[label=(\roman*)]
		\item If $\ipsMean = 1$ (as in SNIPS), then $\est{\droValue}(\pi) \leq 1$.
		\item If $\ipsMean < 1$, then $ \delta < -\log \ipsMean$ if and only if $\est{\droValue} = \infty$ (hence also $\est{\alpha} = \infty$).
		\item If $\minRewardIpsMean < 1$, then $-\log(\minRewardIpsMean) < \delta$ if and only if $\est{\alpha} = 0$ (hence also $\est{\droValue} = m$).
	\end{enumerate}
	
	Graphically, the number of line for $\delta$ looks like:
	\begin{center}
		\begin{tikzpicture}[
			    letter/.style={circle, minimum size=5pt, inner sep=0, outer sep=0, fill=black, label=below:#1},
			    number/.style={fill=white, pos=.5}
			]
			
			\draw (0,0) -- 
			node(zero)[letter=$0$,pos=.1]{}
			node(inf_bdry)[letter=$\max\braces{-\log \ipsMean, 0}$,pos=.4]{}
			node(zero_bdry)[letter=$-\log \minRewardIpsMean$,pos=.7]{}
			node(inf)[letter=$\infty$,pos=1.0]{}
			(10,0);

			\draw [
			    thick,
			    decoration={
				        brace,
				        mirror,
				        raise=0.8cm,
				    },
			    decorate
			] (zero) -- (inf_bdry)
			node [pos=0.5, anchor=north, yshift=-0.8cm] {$\hat\alpha^* = \infty$};
			
			\draw [
			    thick,
			    decoration={
				        brace,
				        mirror,
				        raise=0.8cm,
				    },
			    decorate
			] (zero) -- (inf_bdry)
			node [pos=0.5, anchor=north, yshift=-0.8cm] {$\hat\alpha^* = \infty$};
			\draw [
			    thick,
			    decoration={
				        brace,
				        mirror,
				        raise=0.8cm,
				    },
			    decorate
			] (inf_bdry) -- (zero_bdry)
			node [pos=0.5, anchor=north, yshift=-0.8cm] {$\hat\alpha^* \in (0,\infty)$};
			\draw [
			    thick,
			    decoration={
				        brace,
				        mirror,
				        raise=0.8cm,
				    },
			    decorate
			] (zero_bdry) -- (inf)
			node [pos=0.5, anchor=north, yshift=-0.8cm] {$\hat\alpha^* = 0$};
			\end{tikzpicture}
		\end{center}
\end{theorem}
Case (i) implies that self-normalization is sufficient to avoid blow-up, which is why SNIPS seems to be more stable in practice.
The degenerate case of (ii) occurs when $\delta$ or the propensity ratios are small, and estimation error becomes infinity.
Case (iii), while less degenerate than (ii), is also a degenerate case since we know $\alpha\opt > 0$, so $\wh\alpha = 0$ is not even feasible and provides no useful information about the true value of $\alpha\opt$.

\begin{proof}[Proof of \cref{thm:ips_solution_characterization}]
	First, note that $\est{\phi}(\pi, \alpha)$ is concave in $\alpha$.
	As shown in \cref{sec:kl_calculations}, one can calculate the second derivative to be 
	\[
		\frac{1}{\alpha^3 \prns{ \frac{1}{n} \sum_i w_i \exp(-r_i/\alpha) } } \bracks{ \frac{ \prns{ \frac{1}{n} \sum_i w_i r_i \exp(-r_i/\alpha) }^2 }{ \frac{1}{n} \sum_i w_i \exp(-r_i/\alpha) } - \prns{ \frac{1}{n} \sum_i w_i r_i^2 \exp(-r_i/\alpha)) } }
	\]
	Without changing the sign, give a $\frac{1}{S_w}$ factor to the quantity inside the brackets so that $\frac{1}{n} \sum w_i$ becomes $\frac{1}{n S_w} \sum w_i$. Then the same Cauchy-Schwarz reasoning from Lemma 2 of \citep{si2020distributional} concludes that the whole quantity is non-negative, and strictly positive iff there are two different $r_i$'s. \\
	\noindent \tb{Proof of (i): } \\
	Since $\ipsMean = 1$, the $w_i$ form an empirical distribution, which is bounded by Jensen's inequality
	\begin{align*}
		\widehat \phi(\pi,\alpha) \leq \sup_{\alpha > 0} -\alpha \prns{ \frac{1}{n} \sum_{i=1}^n w_i (-r_i/\alpha) } - \alpha \delta \leq \sum_{i=1}^n w_i r_i \leq 1
	\end{align*}
	
	\noindent \tb{Proof of (ii): } \\
	If $\ipsMean \geq 1$, the claim is vacuous, so let $\ipsMean < 1$. By concavity, $\lim_{\alpha \to \infty} \widehat \phi(\pi,\alpha) = \infty$ is equivalent to $\lim_{\alpha \to \infty} \partialwrt{\alpha} \widehat \phi(\pi, \alpha) > \epsilon$ for some $\epsilon > 0$. The limit of the derivative can be calculated explicitly to be $-\log \ipsMean - \delta$:
	\begin{align}
		\lim_{\alpha \to \infty} -\log \prns{\frac{1}{N} \sum_{i=1}^N w_i \exp(-r_i/\alpha)} - \frac{\sum_{i=1}^N w_i (r_i/\alpha) \exp(-r_i/\alpha)}{\sum_{i=1}^N w_i \exp(-r_i/\alpha)} - \delta = -\log \ipsMean - 0 - \delta \label{eq:phi_grad_limit}
	\end{align}
	
	To see the forward direction, if $\delta < -\log \ipsMean$, then \cref{eq:phi_grad_limit} is at least $\epsilon = \frac{-\log \ipsMean - \delta}{2} > 0$, implying $\est{\alpha} = \infty$. \\
	For the converse, suppose $\est{\alpha} = \infty$, which implies \cref{eq:phi_grad_limit} is at least some $\epsilon > 0$. Then clearly $\delta < -\log \ipsMean - \epsilon < -\log \ipsMean$.
	
	\noindent \tb{Proof of (iii): } \\
	Again leveraging concavity, the idea is that $\est{\phi}(\pi,\alpha)$ achieves sup at $\est{\alpha} = 0$ if and only if the gradient w.r.t. $\alpha$ at 0 is negative. We can calculate the limit explicitly to be $-\log \minRewardIpsMean - \delta$. 
	
	Concretely, consider the limit of $\alpha \to 0^+$. 
	\begin{align*}
		\lim_{\alpha \to 0^+} -\log \prns{\frac{1}{N} \sum_{i=1}^N w_i \exp(-r_i/\alpha)} - \frac{\sum_{i=1}^N w_i (r_i/\alpha) \exp(-r_i/\alpha)}{\sum_{i=1}^N w_i \exp(-r_i/\alpha)} - \delta 
	\end{align*}
	
	Let $m = \min_{i} r_i \geq 0$ be the minimum logged reward. Then we have
	\begin{align*}
		&\lim_{\alpha \to 0^+} \log \prns{\frac{1}{N} \sum_{i=1}^N w_i \exp(-r_i/\alpha)} \\
		&= \lim_{\alpha \to 0^+} -\frac{m}{\alpha} + \log \prns{\frac{1}{N} \prns{\sum_{r_i = m} w_i + \sum_{r_i > m}^N w_i \exp((m-r_i)/\alpha)}} \\
		&= \lim_{\alpha \to 0^+} -\frac{m}{\alpha} + \log\prns{\minRewardIpsMean},
	\end{align*}
	and
	\begin{align*}
		&\lim_{\alpha \to 0^+} \frac{\sum_{i=1}^N w_i (r_i/\alpha) \exp(-r_i/\alpha)}{\sum_{i=1}^N w_i \exp(-r_i/\alpha)} \\
		&= \lim_{\alpha \to 0^+} \frac{\sum_{r_i = m} w_i (m/\alpha) + \sum_{r_i > m} w_i (r_i/\alpha) \exp((m-r_i)/\alpha)}{\sum_{r_i = m} w_i + \sum_{r_i > m} w_i \exp((m-r_i)/\alpha)} \\
		&=\lim_{\alpha \to 0^+} \frac{m}{\alpha}
	\end{align*}
	since $m-r_i < 0$ and so $\lim_{\alpha \to 0^+} \exp((m-r_i)/\alpha) = \exp(-\infty) = 0$.
	Putting these two together, we get 
	\begin{align*}
		\lim_{\alpha \to 0^+} \partialwrt{\alpha} \widehat \phi(\pi, \alpha) = -\log\prns{\minRewardIpsMean} - \delta
	\end{align*}
	Thus, the limit is negative if and only if $\delta > -\log(\minRewardIpsMean)$, as desired.
\end{proof}

\section{Proofs for LDR$^2$OPE}

\subsection{Generic Bandit Moment Equations}
Recall the proposed target equation for bandit feedback by \citep{kallus2019localized} (see Equation (5)),
\begin{align}
    \Eb{U(Y(1); \theta_1) + V(\theta_2)} = 0,  \label{eq:ldml_moment_bandit_eq5}
\end{align}
where $U(\cdot; \theta_1)$ and $V(\theta_2)$ were arbitrary functions that satisfied the conditions of Theorem 3 of \citep{kallus2019localized}.
Note that $V$ only depends on $\theta_2$.
While \cref{eq:ldml_moment_bandit_eq5} captured Quantile Treatment Effect (QTE) and Conditional Value at Risk (CVaR) (which is equivalent to DRO under $\|\cdot\|_\infty$), it is not expressive enough to capture the DROPE objective for KL or $f$-divergences.

We first state a generic moment condition that slightly generalizes \cref{eq:ldml_moment_bandit_eq5} in two ways: (1) we will allow $V$ to also depend on $\theta_1$, and (b) we will allow for stochastic multi-action policies, rather than restricting to binary, deterministic policies.
Our target moment equation is
\begin{align}
	\Eb{U(\piR; \theta_1) + V(\theta)} = 0.  \label{eq:general_moment_bandit}
\end{align}
Since we only have access to $Z=(S,A,R)$, the corresponding orthogonal $\psi$ is,
\begin{align}
	\psi(z; \theta, \eta_1(z; \theta_1), \eta_2(z)) &= \frac{\pi(a|s)}{\eta_2(s, a)} \prns{U(r; \theta_1) - \eta_1(s, a; \theta_1)} + \E_{a \sim \pi(s)}\bracks{\eta_1(s, a; \theta_1)} + V(\theta) \label{eq:ortho_general_moment_bandit} \\
	\text{where } \eta_1\opt(s, a; \theta_1) &= \Eb{U(R; \theta_1) \mid S=s, A=a} \nonumber \\
	\eta_2\opt(s, a) &= \pib(a \mid s) \nonumber
\end{align}
where $\theta\opt$ is the solution to \cref{eq:general_moment_bandit}, $\eta_1$ is the outcome function and $\eta_2$ is the behavior policy.
This is analogous to Equation 9 of \citep{kallus2019localized}, which is the orthogonalized version of \cref{eq:ldml_moment_bandit_eq5}.
It is standard to check that \cref{eq:ortho_general_moment_bandit} satisfies universal orthogonality (see Equation 9 of \citep{kallus2019localized}, or Equation 21 of \citep{foster2019orthogonal}).
Denote the Jacobian and covariance matrices as follows,
\begin{align*}
	&J(\theta') \defeq \partial_{\theta^T} \Eb{ \psi(Z; \theta, \eta_1\opt(Z; \theta_1), \eta_2\opt(Z)) }  \big|_{\theta=\theta'} 
	\qquad \qquad J\opt \defeq J(\theta\opt)
	\\&\psi\opt(Z) \defeq \psi(Z; \theta\opt, \eta_1\opt(Z; \theta_1\opt), \eta_2\opt(Z))
	\\&\Sigma \defeq \E_{\P_0}\bracks{ J^{\star -1} \psi\opt(Z) \psi\opt(Z)\tr J^{\star -\intercal} }
\end{align*}
By replacing $V(\theta_2)$ by $V(\theta)$ and changing $Y(1)$ for $\piR$, we arrive at an exact analog of Theorem 3 of \citep{kallus2019localized}, which we state for completeness.

Let $\mP_N$ denote a sequence of models for the data generating distribution. 
Let $\mT_N$ be the set of possible nuisance realizations.
We use $x_j$ to denote the $j$-th component of a vector $x$. For example, $\eta_{1,j}$ denotes the $j$-th component of $\eta_1$.

\begin{assumption}[Regularity of Estimating Equations]\label{asm:regularity_of_estimating_eq}
	Assume there exist positive constants $c_1$ to $c_4$ such that the following conditions hold for all $\P_0 \in \mP_N$:
	\begin{enumerate}[label=(\roman*)]
		\item $\Theta$ is a compact set and it contains a ball of radius $c_1 N^{-1/2} \log N$ centered at $\theta\opt$.
		\item The map $(\theta, \eta_1(\cdot; \theta_1'), \eta_2) \mapsto \E_{\P_0}\bracks{\psi(Z; \theta, \eta_1(Z; \theta_1'), \eta_2(Z)}$ is twice continuously Gateaux-differentiable.
		\item $\Sigma$ satisfies $c_2 \leq \sigma_{min}(\Sigma) \leq \sigma_{max}(\Sigma) \leq c_3$. The lower bound is for invertibility, while the upper bound is for bounded variance.
		\item The nuisance realization set $\mT_N$ contains the true nuisance parameters $(\eta_1\opt(\cdot; \theta_1\opt), \eta_2\opt(\cdot))$. Moreover, the parameter space $\Theta$ is bounded and for each $(\eta_1(\cdot; \theta_1'), \eta_2(\cdot)) \in \mT_N$, the function class $\mF_{\eta, \theta_1'} = \braces{ \psi_j\prns{Z; \theta, \eta_1(Z; \theta_1'), \eta_2(Z)}, j \in [d], \theta \in \Theta }$ is suitably measurable and its uniform covering entropy satisfies the following: for positive constants $a, v$ and $q > 2$,
		\begin{align*}
			\sup_{\mathbb{Q}} \log N\prns{\epsilon \nm{F_{\eta, \theta_1'}}_{\mathbb{Q}, 2}, \mF_{\eta, \theta_1'}, \nm{\cdot}_{\mathbb{Q}, 2}} \leq v\log(a\epsilon), \forall \epsilon \in [0, 1]
		\end{align*}
		where $F_{\eta, \theta_1'}$ is a measurable envelope for $\mF_{\eta, \theta_1'}$ that satisfies $\nm{F_{\eta, \theta_1'}}_{\P, q} \leq c_4$. \\
		Note, if $\mF_{\eta, \theta_1'}$ are Donsker classes, then this condition is satisfied \citep{van2000asymptotic}.
	\end{enumerate}
\end{assumption}

\begin{assumption}[Nuisance Estimation Rates]\label{asm:ldml_nuisance_estimation_rates}
Let $\rho_{\mu, N}, \rho_{\pi, N}, \rho_{\theta, N}$ denote the converge rates. 
Suppose there exists sequence of constants $\Delta_N \to 0$ s.t.  for any $\P_0 \in \mP_N$, w.p. $1-\Delta_N$, the estimates $\prns{ \wh\eta_1^{(k)}(\cdot; \wh\theta^{(k)}_{1, init}), \wh\pib^{(k)} }$ belong to $\mT_N$, and every $j \in [d]$,
\begin{align*}
	&\nmP{\wh\eta^{(k)}_{1,j}(S, A; \wh\theta^{(k)}_{1, init}) - \eta_{1,j}\opt(S, A; \wh\theta^{(k)}_{1, init})} \leq \rho_{\eta_1, N} \\
	&\nmP{ \wh\pib^{(k)}(S, A) - \pib(S, A) } \leq \rho_{\pib, N} \\
	&\nm{ \wh\theta^{(k)}_{1, init} - \theta\opt } \leq \rho_{\theta, N}
\end{align*}
\end{assumption}

\begin{theorem}\label{thm:general_moment_bandit}
	Let $\hat \theta$ be given by applying LDML to \cref{eq:ortho_general_moment_bandit}.
	Suppose \cref{asm:ldml_nuisance_estimation_rates}.
	Suppose there exists positive constants $c_1$ to $c_{10}$ s.t. for any $\P_0 \in \mP_N$, the following holds:
	\begin{enumerate}[label=(\roman*)]
		\item \cref{asm:regularity_of_estimating_eq} with constants $c_1$ to $c_4$
		\item The estimating equation solution approximation error satisfies $\varepsilon_N = \delta_N N^{-1/2}$, where $d_N \to 0$.
		\item Let $\theta \in \Theta$ be arbitrary. For each $j \in [d]$, the map $\theta \mapsto \Eb{U_j(\piR; \theta_1) + V_j(\theta)}$ is differentiable, and each component of its gradient is Lipschitz continuous at $\theta\opt$ with Lipschitz constant $c_5$. Moreover, if $\nm{\theta - \theta\opt} \geq \frac{c_6}{2\sqrt{d} c_5}$, then $2\nm{\Eb{U(\piR; \theta_1) + V(\theta)}} \geq c_7$.
		\item $J\opt = \partial_{\theta^T} \Eb{U(\piR; \theta_1) + V(\theta)}|_{\theta = \theta\opt}$ satisfies that $c_8 \leq \sigma_{min}(J\opt) \leq \sigma_{max}(J\opt) \leq c_9$.
		\item For any $\theta \in \mB\prns{\theta\opt; \frac{4c_{10}\sqrt{d}\rho_{\pi,N}}{\delta_N \piblb}} \cap \Theta, r \in (0, 1)$ and for $j \in [d]$, there exists functions $h_1, h_2$ s.t. $\Eb{h_i(S, A, \theta_1)} < \infty$ for $i \in [2]$, and almost surely
		\begin{align*}
			&\abs{ \partial_r \eta_{1,j}\opt\prns{S, A; \theta_1\opt + r(\theta_1-\theta_1\opt)} } \leq h_1(S, A, \theta_1) \\
			&\abs{ \partial_r^2 \eta_{1,j}\opt\prns{S, A; \theta_1\opt + r(\theta_1 - \theta_1\opt)} } \leq h_2(S, A, \theta_1)
		\end{align*}
		\item For $j \in [d]$:
		\begin{align*}
			&\prns{ \Eb{\eta_{1,j}\opt\prns{S, A; \theta_1}}^2 }^{1/2} \leq c_{10} \\
			&\nm{ \prns{ \Eb{ \partial_{\theta_1} \eta_{1,j}\opt\prns{X, t, \theta_1} }^2 }^{1/2} } \leq c_{10} \\
			&\sigma_{max}\prns{ \Eb{ \partial_{\theta_1} \partial_{\theta_1^T} \eta_{1,j}\opt(S, A; \theta_1) } } \leq c_{10} \\
			&\sigma_{max}\prns{ \Eb{ \partial_{\theta} \partial_{\theta^T} V_j(\theta) } } \leq c_{10}
		\end{align*}
		and for any $\theta \in \mB\prns{ \theta\opt; \max\braces{ \frac{4c_{10}\sqrt{d}\rho_{\pi, N}}{\delta_N \piblb}, \rho_{\theta, N} }}  \cap \Theta$, 
		\begin{align*}
			\prns{ \Eb{ \eta_{1,j}\opt(S, A; \theta_1) - \eta_{1,j}\opt(S, A; \theta_1\opt) }^2 }^{1/2} \leq c_{10} \nm{\theta_1 - \theta_1\opt}
		\end{align*}
		\item $\rho_{\pi, N}(\rho_{\mu, N} + c_{10}\rho_{\theta, n}) \le \frac{\piblb^3}{3}\delta_N N^{-1/2}$, $\rho_{\pi, N} \le \frac{\delta_N^3}{\log N}$, and $\rho_{\mu, N} + c_{10}\rho_{\theta, N} \le \frac{\delta^2_N}{\log N}$, $\delta_N \le \frac{4c_{10}^2\sqrt{d} + 2\piblb}{\piblb^2}$, and $\delta_N \le \min\{\frac{\piblb^2}{8c_{10}^2 d}\log N, \sqrt{\frac{\piblb^3}{2c_{10}\sqrt{d}}}\log^{1/2} N\}$
	\end{enumerate}
	Then, uniformly over $\P_0 \in \mP_N$,
	\begin{align*}
		\sqrt{N} \Sigma^{-1/2} (\hat{\theta} - \theta\opt) = \frac{1}{\sqrt{N}} \sum_{i=1}^N \Sigma^{-1/2} J^{\star -1} \psi(Z_i; \theta\opt, \eta_1\opt(Z_i, \theta_1\opt), \eta_2\opt(Z_i)) + \mO_P(\rho_N) \rightsquigarrow N(0, I_d)
	\end{align*}
	where $\rho_N = o_{\P_0}(1)$ is given in Theorem 1 of \citep{kallus2019localized}.
	Furthermore, $\Sigma$ is the best possible covariance matrix for regular and asymptotic linear (RAL) estimators; that is, for every RAL estimator with $\Sigma'$ covariance matrix, $\Sigma' - \Sigma$ is positive semi-definite \citep{tsiatis2007semiparametric}. 
\end{theorem}
\begin{proof}[Proof of \cref{thm:general_moment_bandit}]
The proof is the same as the proof of Theorem 3 in \citep{kallus2019localized}, except we replace $Y(t)$ by $\piR$ and $V(\theta_2)$ by $V(\theta)$.
\end{proof}

\subsection{Efficiency for DROPE}\label{sec:efficiency-proof}
We now prove \cref{thm:dro_eval_ldml_efficiency} by showing that the specific choice of $U, V$ in \cref{eq:def_drope_uv} is well-behaved, and satisfies the assumptions of \cref{thm:general_moment_bandit}.
Note that \cref{thm:general_moment_bandit} is a uniform guarantee over a family of nominal distributions $\mP_N$. For simplicity, we will take the family of models as a singleton with the nominal data generating process $\mP_N = \braces{\PP_0}$.
This simplifies many of the regularity assumptions, as we will remark in the proof below.
Under these additional regularity conditions (which are standard), our proof is easily extendable to hold uniformly over a family of nominal distributions, which may be beneficial from a finite-sample perspective \citep{kasy2019uniformity}.

\efficiency*
\begin{proof}[Proof of \cref{thm:dro_eval_ldml_efficiency}]
    First, we will list some useful calculations. Then, we will verify the assumptions of \cref{thm:general_moment_bandit}, with slight simplifications since we are showing convergence for a single distribution $\PP_0$, rather than a set of distributions $\mP_N$.

	The function $\alpha \mapsto \Eb{\exp(-\piR/\alpha)}$ is three-times differentiable, and w.p. 1, the random function $\alpha \mapsto \Eb{\exp(-\piR/\alpha)|S, A}$ is three-times differentiable. This follows from Dominated Convergence Theorem, since $\piR^j \exp(-\piR/\alpha), j \in [4]$ is bounded, and so we can pass limits into the expectation.
	Let us denote $\theta\opt = [\alpha\opt, W_0\opt, W_1\opt, \droValue\opt]$.
	\begin{align*}
		\Eb{\psi(Z; \theta, \eta_1\opt(Z;\theta_1), \eta_2\opt)}
		&= \Eb{U(\piR; \theta_1) + V(\theta)} \\
		&= \begin{bmatrix}
			\Eb{\exp(-\piR/\alpha)} - W_0 \\
			\Eb{\piR \exp(-\piR/\alpha)} - W_1 \\
			-\delta - \log W_0 - \frac{W_1}{\alpha W_0} \\
			-\droValue - \alpha \log W_0 - \alpha \delta
		\end{bmatrix},
	\end{align*}
	and, the Jacobian is the following (recall the order of $\theta = \bracks{\alpha, W_0, W_1, \droValue}$):
	\begin{align}
		J(\theta) = \begin{bmatrix}
			\frac{1}{\alpha^2} \Eb{\piR \exp(-\piR/\alpha)} & -1 & 0 & 0 \\
			\frac{1}{\alpha^2} \Eb{\piR^2 \exp(-\piR/\alpha)} & 0 & -1 & 0 \\
			\frac{W_1}{\alpha^2 W_0} & -\frac{1}{W_0} + \frac{W_1}{\alpha W_0^2} & -\frac{1}{\alpha W_0} & 0 \\
			-\log W_0 - \delta & -\frac{\alpha}{W_0} & 0 & -1
		\end{bmatrix}.  \label{eq:ldrope-jacobian-matrix}
	\end{align}
	Hence, substituting in the optimal value, we have
	\begin{align*}
		J\opt = J(\theta\opt) = 
		\begin{bmatrix}
			\frac{W_1\opt}{\alpha^{\star 2}} & -1 & 0 & 0 \\
			\frac{W_2\opt}{\alpha^{\star 2}} & 0 & -1 & 0 \\
			\frac{W_1\opt}{\alpha^{\star 2} W_0\opt} & -\frac{1}{W_0\opt} + \frac{W_1\opt}{\alpha\opt W_0^{\star 2}} & -\frac{1}{\alpha\opt W_0\opt} & 0 \\
			-\log W_0\opt - \delta & -\frac{\alpha\opt}{W_0\opt} & 0 & -1
		\end{bmatrix}.
	\end{align*}
	
	By Cramer's rule, we now show that $\det J\opt = -\phi''(\pi, \alpha\opt)$:
	\begin{align*}
		-\det J\opt
		&= \det 
		\begin{bmatrix}
			\frac{W_1\opt}{\alpha^{\star 2}} & -1 & 0 \\
			\frac{W_2\opt}{\alpha^{\star 2}} & 0 & -1 \\
			\frac{W_1\opt}{\alpha^{\star 2} W_0\opt} & -\frac{1}{W_0\opt} + \frac{W_1\opt}{\alpha\opt W_0^{\star 2}} & -\frac{1}{\alpha\opt W_0\opt}
		\end{bmatrix}
		\\
		&= \frac{W_1\opt}{\alpha^{\star 2}} \cdot \prns{-\frac{1}{W_0\opt} + \frac{W_1\opt}{\alpha\opt W_0^{\star 2}}} - \frac{W_2\opt}{\alpha^{\star 2}} \cdot \frac{1}{\alpha\opt W_0\opt} + \frac{W_1\opt}{\alpha^{\star 2} W_0\opt} \\
		&= \frac{1}{\alpha^{\star 3} W_0\opt} \prns{ \frac{ W_1^{\star 2}}{W_0\opt} - W_2\opt } \\
		&= \phi''(\alpha\opt)
	\end{align*}
	
	Since $\phi$ is strictly concave, we have $\det J\opt > 0$ and so $J\opt$ is invertible.
	We can compute the inverse by querying ``invert \{\{A, -1, 0, 0\}, \{B, 0, -1, 0\}, \{C, D, E, 0\}, \{F, G, 0, -1\}\}" on Wolfram Alpha,
	\begin{align*}
		M = 
		\begin{bmatrix}
			A & -1 & 0 & 0 \\
			B & 0 & -1 & 0 \\
			C & D & E & 0 \\
			F & G & 0 & -1
		\end{bmatrix}
		\qquad
		M^{-1} =
		\frac{1}{S}
		\begin{bmatrix}
			D & E & 1 & 0 \\
			-(B E + C) & A E & A & 0 \\
			B D & -(A D + C) & B & 0 \\
			D F - G (B E + C) & E (A G + F) & A G + F & -S
		\end{bmatrix}
	\end{align*}
	where $S = AD + BE + C = \phi''(\pi,\alpha\opt)$ since
	\begin{align*}
		AD + BE + C
		&= \frac{W_1\opt}{\alpha^{\star 2}} \prns{-\frac{1}{W_0\opt} + \frac{W_1\opt}{\alpha\opt W_0^{\star 2}}} + \frac{1}{\alpha^{\star 2} W_0\opt} \prns{ -\frac{W_2\opt}{\alpha\opt} + W_1\opt } \\
		&= \frac{1}{\alpha^{\star 3} W_0\opt} \prns{ \frac{W_1^{\star 2}}{W_0\opt} - W_2\opt } \\
		&= \phi''(\pi, \alpha\opt)
	\end{align*}
	Thus
	\begin{align*}
		J^{\star -1} = 
		\frac{1}{\phi''(\pi, \alpha\opt)}
		\begin{bmatrix}
			-\frac{1}{W_0\opt} + \frac{W_1\opt}{\alpha\opt W_0^{\star 2}} & -\frac{1}{\alpha\opt W_0\opt} & 1 & 0 \\
			\frac{W_2\opt/\alpha\opt - W_1\opt}{\alpha^{\star 2} W_0\opt} & -\frac{W_1\opt}{\alpha^{\star 3} W_0\opt} & \frac{W_1\opt}{\alpha^{\star 2}} & 0 \\
			\frac{W_2\opt}{\alpha^{\star 3} W_0\opt} \prns{\frac{W_1\opt}{W_0\opt} - \alpha\opt} & -\frac{W_1^{\star 2}}{ \alpha^{\star 3} W_0^{\star 2} } & \frac{W_2\opt}{\alpha^{\star 2}} & 0 \\
			X_1 & \frac{X_2}{\alpha\opt W_0\opt} & -X_2 & -\phi''(\pi,\alpha\opt)
		\end{bmatrix}
	\end{align*}
	where $X_1 = \prns{\frac{W_1\opt}{\alpha\opt W_0^{\star 2}} - \frac{1}{W_0\opt}} \prns{-\log W_0\opt - \delta} - \frac{1}{\alpha\opt W_0^{\star 2}} \prns{W_1\opt - W_2\opt/\alpha\opt}$ and $X_2 = W_1\opt/(\alpha\opt W_0\opt) + \log W_0\opt + \delta$.
	
	\noindent \tb{Verifying condition (i) of \cref{thm:general_moment_bandit}:} \\
	We now verify \cref{asm:regularity_of_estimating_eq} (i)-(iv).
	For (i), the parameter values of $\theta = [\alpha, W_0, W_1, \droValue]$ are bounded, and by considering the closure of the set, we have a compact $\Theta$. Since $\alpha\opt > 0$, we can set $c_1$ small enough so that the ball of radius $c_1$ exists at $\theta\opt$.
	As shown already, $\Eb{\psi}$ is three-times continuously differentiable, giving (ii).
	Since entries of $J^{\star -1}$ and $\psi\opt(Z)$ are upper-bounded, this implies that $\sigma_{max}(\Sigma)$ is upper-bounded.
	Since $\mP_N$ is a singleton, we actually do not need that $\Sigma$ be invertible and directly apply Central Limit Theorem (\cref{thm:general_moment_bandit} needed to invert $\Sigma$ in the statement to make the target distribution fixed as a standard normal).
	Finally, observe that $\braces{r \mapsto \exp(-r/\alpha) - W_0 \mid \alpha > 0, W_0 \in \Rl}, \braces{r \mapsto r\exp(-r/\alpha) - W_1 \mid \alpha > 0, W_1 \in \Rl }$ are Donsker classes. This implies the metric entropy codition in (iv), see \citep{vandervaart1996weak,van2000asymptotic}.
	
	\noindent \tb{Verifying condition (ii) of \cref{thm:general_moment_bandit}:} \\
	The moment condition can be solved exactly, for instance using Newton's method for an initial estimate $\alpha_0$ sufficiently close to $\alpha^\star$. Hence, $\epsilon_N = 0$.
	In practice, we found a good heuristic to seed $\alpha_0$ as the average of the $K$ localized estimates $\wh\alpha^{(k)}_{init}$.
	
	\noindent \tb{Verifying condition (iii) of \cref{thm:general_moment_bandit}:} \\
	By visual inspection of the $J(\theta)$ matrix (\cref{eq:ldrope-jacobian-matrix}), and by differentiability of $\alpha \mapsto \Eb{\piR^j \exp(-\piR/\alpha)}, j=1,2$, we have that each component of $J(\theta)$ is Lipschitz, with Lipschitz constant at most $L := \frac{1}{\alpha^{\star 3} W_0^{\star 3}}$.
	We also have consistency, i.e. if $\theta \neq \theta\opt$, then $\|\Eb{U(\piR; \theta_1)+V(\theta)}\| > 0$, since $\phi(\pi,\cdot)$ is strictly concave. 
	Indeed, if $\alpha \neq \alpha\opt$, we have $|\phi'(\alpha)| > 0$, which is the third component of $\Eb{U(\piR; \theta_1)+V(\theta)}$.
	And if $\alpha = \alpha\opt$ but $W_0 \neq W_0\opt$, then $|\Eb{\exp(-R/\alpha\opt)} - W_0| = |W_0\opt - W_0| > 0$, which is the first component of $\Eb{U(\piR; \theta_1)+V(\theta)}$. The same reasoning applies for second and fourth components.
	
	\noindent \tb{Verifying condition (iv) of \cref{thm:general_moment_bandit}:} \\
	Here, we want to show that singular values of $J\opt$ are lower and upper bounded.
	The maximum of the entries of $J\opt$ is an upper bound for $\sigma_{max}(J\opt)$ and the inverse of the maximum of all the entries for $J^{\star -1}$ is a lower bound for $\sigma_{min}(J\opt)$.
	Since both $\alpha\opt$ and $\wlb$ are positive and finite, we have that the lower bound is positive, and the upper bound is positive and finite.
	
	\noindent \tb{Verifying conditions (v) and (vi) of \cref{thm:general_moment_bandit}:} \\
	For (vi),
	\begin{align}
		&\Eb{\eta_{1,1}\opt(S, A; \theta_1)}^2 = \Eb{\exp(-R/\alpha)}^2 \leq 1 \\
		&\Eb{\eta_{1,2}\opt(S, A; \theta_1)}^2 = \Eb{R\exp(-R/\alpha)}^2 \leq 1
	\end{align}
	
	The following calculations will be useful for both (v) and (vi). Let $r \in [0, 1]$,
	\begin{align*}
		\abs{ \partial_r \eta_{1,1}\opt(s, a; \theta_1\opt + r(\theta_1 - \theta_1\opt) }
		&=\abs{ \partial_r \Eb{ \exp\prns{-R/(\alpha\opt + r(\alpha - \alpha\opt))} | S=s, A=a} } 
		\\&=\abs{ \frac{\Eb{ R\exp\prns{-R/(\alpha\opt + r(\alpha - \alpha\opt))} | S=s, A=a}}{(\alpha\opt + r(\alpha - \alpha\opt))^2} \prns{\alpha - \alpha\opt} } 
		\\&\leq \abs{\frac{1}{(\alpha\opt + r(\alpha-\alpha\opt)^2}} \cdot \abs{\alpha - \alpha\opt} 
		\\
		\abs{ \partial^2_r \eta_{1,1}\opt(s, a; \theta_1\opt + r(\theta_1 - \theta_1\opt) }
		&=\Bigg| 
		\frac{\Eb{ R^2\exp\prns{-R/(\alpha\opt + r(\alpha - \alpha\opt))} | S=s, A=a}}{(\alpha\opt + r(\alpha-\alpha\opt))^4} 
		\\&+ \frac{2\Eb{ R\exp\prns{-R/(\alpha\opt + r(\alpha - \alpha\opt))} | S=s, A=a} }{(\alpha\opt + r(\alpha-\alpha\opt))^3} \Bigg| \cdot \abs{\alpha-\alpha\opt}^2 
		\\&\leq \prns{\frac{1}{(\alpha\opt + r(\alpha-\alpha\opt))^4} + \frac{2}{(\alpha\opt + r(\alpha-\alpha\opt))^3}} \abs{\alpha - \alpha\opt}^2
	\end{align*}
	$\eta_{1,2}\opt$ has an additional $R$ multiplied, but since $R \in [0, 1]$, the bound is the same. So,
	\begin{align*}
		&\abs{ \partial_r \eta_{1,2}\opt(s, a; \theta_1\opt + r(\theta_1 - \theta_1\opt) } \leq \abs{\frac{1}{(\alpha\opt + r(\alpha-\alpha\opt)^2}} \cdot \abs{\alpha - \alpha\opt} \\
		&\abs{ \partial^2_r \eta_{1,2}\opt(s, a; \theta_1\opt + r(\theta_1 - \theta_1\opt) } \leq \prns{\frac{1}{(\alpha\opt + r(\alpha-\alpha\opt))^4} + \frac{2}{(\alpha\opt + r(\alpha-\alpha\opt))^3}} \abs{\alpha - \alpha\opt}^2
	\end{align*}
	If $\theta$ is sufficiently close to $\theta\opt$ (when $\rho_{\pi, N}$ is small enough, i.e. when $N$ is large enough), we have that $\alpha > \alpha\opt/2 > 0$. Hence, $\partial_r \eta_{1,j}\opt$ and $\partial_r^2 \eta_{1,j}\opt$ are upper bounded by $\frac{3 \cdot 2 \cdot 4}{\alpha^{\star 2}}$.
	This fully verifies (v), as well as most of (vi). 
	
    Let $N$ be sufficiently large s.t. $\max\{ \rho_{\pi,N}, \rho_{\theta, N} \} < \alpha\opt/2$, so for any $\theta$ close enough to $\theta\opt$ (so that $\alpha > \alpha\opt/2$), we have
	\begin{align}
		\Eb{ \eta_{1,1}\opt(S, A; \theta_1) - \mu_1\opt(S, A; \theta_1\opt) }^2
		= \Eb{ \exp(-R/\alpha) - \exp(-R/\alpha\opt) }^2
		\leq \frac{4}{\alpha^{\star 2}} \abs{\alpha-\alpha\opt}.
	\end{align}
	The $\eta_{1,2}\opt$ case is analogous, which concludes all of (vi).
	
	\noindent \tb{Semiparametric efficiency for $\droValue$ and $\alpha$:} \\
	Since we've verified all the Assumptions of \cref{thm:general_moment_bandit}, we have that $\wh\theta$ achieves semiparametric efficiency.
	Then, by Theorems 25.20, 25.21 of \citep{van2000asymptotic}, and the fact that indexing is a cone-shaped function, we also have semiparametric efficiency for each index of $\theta\opt$, in particular $\droValue\opt$ and $\alpha\opt$.
\end{proof}

\subsection{Newton-Raphson Method} \label{sec:newtons_dro_eval}
In this section, we use Newton-Raphson to with projections to $\Rl^+$ to solve the moment equation in \cref{algo:dro_eval_ldml}, which recall is
\begin{align*}
    M(\alpha) := -\delta - \log(\est{W_0}(\alpha)) - \frac{\est{W_1}(\alpha)}{\alpha \cdot \est{W_0}(\alpha)} = 0,
\end{align*}
where $\wh W_j$ is defined in \cref{algo:dro_eval_ldml}.

First, initialize $\alpha_0 = \frac{1}{K} \sum_{k=1}^K \alpha_{init}^{(k)}$ to be the average of the outputs of the subroutine calls to cross-fitted SNIPS (since Newton's method should be seeded with something close to $\alpha\opt$). 
Then, take the following update steps until convergence (i.e. $|\alpha_{t+1}-\alpha_t| < \eps$),
\begin{align*}
	&\alpha_{t+1} = \alpha_t - M(\alpha_t) / M'(\alpha_t) \\
	\text{where }
	&M'(\alpha) = -\frac{\wh W_0'(\alpha)}{\wh W_0(\alpha)} - \frac{\wh W_1'(\alpha) \cdot \alpha \wh W_0(\alpha) - \wh W_1(\alpha) \cdot \prns{\wh W_0(\alpha) + \alpha \wh W_0'(\alpha)} }{\prns{ \alpha \wh W_0(\alpha) }^2 },
\end{align*}
where the derivatives only include the $\alpha$-dependent IPS part of $\wh W_j$, so
\begin{align}
    &\wh W_0'(\alpha) = \frac{1}{N} \sum_{k=1}^K \sum_{i \in D_k} \frac{\pi(a_i\mid s_i)}{\wh\pib(s_i, a_i)} \exp(-r_i/\alpha) \frac{r_i}{\alpha^2}, \label{eq:dro_eval_moment_W_deriv_wrt_alpha}
	\\&\wh W_1'(\alpha) = \frac{1}{N} \sum_{k=1}^K \sum_{i \in D_k} \frac{\pi(a_i\mid s_i)}{\wh\pib(s_i, a_i)} \exp(-r_i/\alpha) \frac{r_i^2}{\alpha^2}. \label{eq:dro_eval_moment_dW_deriv_wrt_alpha}
\end{align}
If the update takes $\alpha_{t+1}$ outside the feasible region $[0, 1/\delta]$, then project it back.

\subsection{Multidimensional Newton's Method} \label{sec:multidimensional_newtons_dro_eval}
Instead of thinking about the moment condition as a function of $\alpha$, we can think about it as a function of $\theta$, and perform multidimensional Newton's method. This is the formulation that is most natural from applying LDML, with the following multidimensional condition,
\begin{align*}
	\psi(\theta) = 
	\begin{bmatrix}
		-W_0 + \wh W_0(\alpha) \\
		-W_1 + \wh W_1(\alpha) \\
		-\delta - \log W_0 - \frac{W_1}{\alpha W_0} \\
		-\droValue - \alpha \log W_0 - \alpha \delta
	\end{bmatrix}
	= \textbf{0}.
\end{align*}

We'll need to calculate the Jacobian matrix. The only difference from \cref{eq:ldrope-jacobian-matrix} is that the entries with $\Eb{\cdot}$ are replaced with IPS estimates. In other words,
\begin{align*}
    \wh J(\theta) := 
    \begin{bmatrix}
		\wh W_0'(\alpha) & -1 & 0 & 0 \\
		\wh W_1'(\alpha) & 0 & -1 & 0 \\
		\frac{W_1}{\alpha^2 W_0} & -\frac{1}{W_0} + \frac{W_1}{\alpha W_0^2} & -\frac{1}{\alpha W_0} & 0 \\
		-\log W_0 - \delta & -\frac{\alpha}{W_0} & 0 & -1
	\end{bmatrix},
\end{align*}
where $\wh W_j', j=0,1$ are calculated in\cref{eq:dro_eval_moment_W_deriv_wrt_alpha,eq:dro_eval_moment_dW_deriv_wrt_alpha}.
Now, we can apply the following updates, until convergence (i.e. $\|\theta_{t+1} - \theta_t\| < \eps$):
\begin{align*}
	\theta_{t+1} &= \theta_t - \wh J(\theta_t)^{-1} \psi(\theta_t).
\end{align*}

\begin{remark}
We empirically tested both (1D) Newton-Raphson and the multidimensional Newton's method and found no significant difference in the MSE or final values of $\alpha$. There may be some small sample differences but when $N \geq 1024$, both approaches essentially gave the exact same solutions.
\end{remark}

\subsection{Runtime Analysis}\label{sec:runtime-analysis-ldrope}
In this section, we analyze the total runtime (a.k.a. work) and parallelized runtime (a.k.a. span) of Localized Doubly Robust \cref{algo:dro_eval_ldml}. 

Let $\workOf{\pib}{N}, \workOf{f}{N}, \workOf{init}{N}$ respectively denote the work of fitting $\pib$, $f_j$ (for both $j = 0, 1$),  and running $\InitialEstimate$, on an input dataset of size $N$. 
Let $\spanOf{\pib}{N}, \spanOf{f}{N}, \spanOf{init}{N}$ denote the span analogs of the above. 
Suppose these are non-decreasing functions; that is, having more data will only increase training work/span. 

Note that, assuming inference of $\wh\pib^{(k)}, \wh f_j^{(k)}$ takes constant time on a single sample, solving the moment equation takes $\mO(N)$ work and $\mO(\log(N))$ span.

A single run of LDR$^2$OPE has work/span bounded by
\begin{align*}
	&\workOf{\mathrm{\text{LDR}^2\text{OPE}}}{N} = \mO\prns{ K \prns{ \workOf{\pib}{\frac{(K-1)N}{K}} + \workOf{f}{\frac{(K-1)N}{2K}} + \workOf{init}{\frac{(K-1)N}{2K}} } } \\
	&\spanOf{\mathrm{\text{LDR}^2\text{OPE}}}{N} = \mO\prns{ \max\prns{\spanOf{\pib}{\frac{(K-1)N}{K}}, \spanOf{f}{\frac{(K-1)N}{2K}} 
		+ \spanOf{init}{\frac{(K-1)N}{2K}}} }
\end{align*}
where the work expression follows directly from examining the sizes of the datasets on each iteration. The span expression doesn't have a $K$ multiplier since each cross-fitting step can be parallelized. Also, fitting $\pib$ can be done in parallel when running $\InitialEstimate$ and then fitting $f_j$ (which depends on the output of $\InitialEstimate$).

Now, we analyze the work/span of $m$-recursive runs of LDR$^2$OPE. They satisfy the following recurrences:
\begin{align*}
&\workOf{\mathrm{\text{LDR}^2\text{OPE}},m}{N} = \mO\prns{ K \prns{ \workOf{\pib}{\frac{(K-1)N}{K}}
			+ \workOf{f}{\frac{(K-1)N}{2K}} 
			+ \workOf{\mathrm{\text{LDR}^2\text{OPE}},m-1}{\frac{(K-1)N}{2K}} } } \\
&\spanOf{\mathrm{\text{LDR}^2\text{OPE}},m}{N} = \mO\prns{
		\spanOf{\pib}{\frac{(K-1)N}{K}} + \spanOf{f}{\frac{(K-1)N}{2K}} 
			+ \spanOf{\mathrm{\text{LDR}^2\text{OPE}},m-1}{\frac{(K-1)N}{2K}} }
\end{align*}
where we upper bounded $\max$ by $+$ for span to simplify the solution. The recurrences solve to,
\begin{align*}
&\workOf{\mathrm{\text{LDR}^2\text{OPE}},m}{N} = \mO\prns{
	\sum_{t=1}^m K^t \workOf{\pib}{\frac{(K-1)^t N}{2^{t-1} K^t}} 
	+ \sum_{t=1}^m K^t \workOf{f}{\frac{(K-1)^t N}{2^t K^t}} 
	+ K^m \workOf{init}{\frac{(K-1)^m N}{2^m K^m}}
} \\
&\spanOf{\mathrm{\text{LDR}^2\text{OPE}},m}{N} = \mO\prns{
	\sum_{t=1}^m \spanOf{\pib}{\frac{(K-1)^t N}{2^{t-1}K^t}}
	+ \sum_{t=1}^m \spanOf{f}{\frac{(K-1)^t N}{2^t K^t}} 
	+ \spanOf{init}{\frac{(K-1)^m N}{2^m K^m}}
}
\end{align*}

Since fitting the nuisances $\pib, f_j$ are standard regression tasks, there are many poly-time (many of which are linear) learning algorithms. For example, linear regression, neural nets trained with SGD, and XGBoost can all take $\mO\prns{N}$ work to train. 
To keep analysis generic, suppose that $\workOf{\pib}{N}, \workOf{f}{N}, \spanOf{\pib}{N}, \spanOf{\pib}{N} = \wt\mO\prns{N^p}$ for some $p \geq 1$.
Then cross-fitted SNIPS has work and span
\begin{align*}
&\workOf{\mathrm{xfit-snips}}{N} = \mO\prns{
	K \prns{ \workOf{\pib}{\frac{(K-1)N}{K}}
		+ \workOf{f}{\frac{(K-1)N}{K}} } 
} = \wt\mO\prns{K N^p} \\
&\spanOf{\mathrm{xfit-snips}}{N} = \mO\prns{
	\spanOf{\pib}{\frac{(K-1)N}{K}}
	+ \spanOf{f}{\frac{(K-1)N}{K}}
} = \wt\mO\prns{N^p}
\end{align*}
So starting with $\InitialEstimate$ being cross-fitted SNIPS, and recursively running LDR$^2$OPE $m$ times has work and span
\begin{align*}
&\workOf{\mathrm{\text{LDR}^2\text{OPE}},m}{N} = \wt\mO\prns{K^{m + 1} N^p} \\
&\spanOf{\mathrm{\text{LDR}^2\text{OPE}},m}{N} = \wt\mO\prns{mN^p}
\end{align*}

\section{Proof of Regret Guarantees for DROPL}\label{sec:learning_guarantees_proofs}

In our analysis, we assume that the estimated nuisances fall into their appropriate ranges: $\wh f_0^{(k)}(s, a; \alpha) \in (0, 1], \wh\pib^{(k)}(s,a) \in [\piblb, 1]$. This is without loss of generality since it can always be satisfied by clipping.

\begin{assumption}\label{asm:sufficiently_large_n}
We suppose that $N$ is sufficiently large. Specifically, for the $\beta$ from \cref{thm:cts_reward_dr_regret} Let $\beta \in (0, 1)$, we need the following to hold
\begin{align*}
& \frac{\wlb}{2} \geq \frac{288}{\piblb \sqrt{N}} \prns{\kappa(\Pi) + L\alphaub \vee 1} + \frac{4}{\piblb} \log^{1/2}(1/\beta),
\\& \frac{\wlb}{4} \geq \frac{384\sqrt{K}}{\piblb\sqrt{N}} \prns{ \kappa(\Pi) + L\alphaub \vee 1 } + \frac{8\sqrt{K}\log^{1/2}(K/\beta)}{\piblb} + \frac{\op{Rate}_{\pib}(N,\beta/K) \cdot \op{Rate}_f^{\mathfrak{c}}(N,\beta/K) }{\piblb^2}
\end{align*}
To satisfy both, it suffices to take,
\begin{align*}
    \sqrt{N} \geq \frac{4}{\wlb} \prns{ \frac{384\sqrt{K}}{\piblb} \prns{ \kappa(\Pi) + L\alphaub \vee 1 } + \frac{8\sqrt{K}\log^{1/2}(K/\beta)}{\piblb} + \frac{\op{Rate}_{\pib}(N,\beta/K) \cdot \op{Rate}_f^{\mathfrak{c}}(N,\beta/K)\sqrt{N}}{\piblb^2} },
\end{align*}
Provided that $\op{Rate}_{\pib}(N,\beta) \cdot \op{Rate}_f^{\mathfrak{c}}(N,\beta) \leq o(N^{-1/2})$, it will not be part of the dominant term.
\end{assumption}

\regretbound*
\begin{proof}[Proof of \cref{thm:cts_reward_dr_regret}]
The steps for bounding regret are inspired by uniform coupling arguments bounding OPL regret \citep{athey2021policy,zhou2018offline}.
First, define the \textit{infeasible} CFDR values $W^{DR}$ and $\droValue^{DR}$ (without the hat), with the \textit{true} nuisances; that is, replace $\wh\pib^{(k)}$ and $\wh f_0^{(k)}(\cdot;\alpha)$ in \cref{eq:dro_learning_erm} by the true $\pib$ and $f_0(\cdot;\alpha)$ respectively.
Then, we show two uniform concentrations (with rate $\mO(N^{-1/2})$) simultaneously over $\Pi$ and $\alpha$; \cref{lm:oracle_dr_close_to_truth} concentrates $\droValue$ to $\droValue^{DR}$, and \cref{lm:estimated_dr_close_to_oracle_dr} concentrates $\droValue^{DR}$ to $\est{\droValue}^{DR}$.
So,
\begin{align*}
	\droRegret{\est{\pi}^{DR}}
	&= \droValue(\pi\opt) - \est{\droValue}^{DR}(\pi\opt) + \est{\droValue}^{DR}(\pi\opt) - \droValue(\est{\pi}^{DR})
	\\&\leq \droValue(\pi\opt) - \est{\droValue}^{DR}(\pi\opt) + \est{\droValue}^{DR}(\est{\pi}^{DR}) - \droValue(\est{\pi}^{DR}) 
	\\&\leq 2\sup_{\pi \in \Pi} \abs{\droValue(\pi) - \est{\droValue}^{DR}(\pi)} 
	\\&\leq 2\sup_{\pi \in \Pi} \abs{\droValue(\pi) - \droValue^{DR}(\pi)} 
	+ 2\sup_{\pi \in \Pi} \abs{\droValue^{DR}(\pi) - \est{\droValue}^{DR}(\pi)}
	\\&\leq \frac{2\alphaub}{\wlb} \prns{ \frac{288}{\piblb \sqrt{N}} \prns{\kappa(\Pi) + \frac{\alphaub}{\alphalb^2}} + \frac{4}{\piblb \sqrt{N}} \log^{1/2}(1/\beta) } 
	\\&+ \frac{4\alphaub}{\wlb} \prns{ \frac{384}{\piblb\sqrt{N/K}} \Bigg( \kappa(\Pi) + \frac{\alphaub}{\alphalb^2} } + \frac{8\log^{1/2}(K/\beta)}{\piblb\sqrt{N/K}} 
	+ \frac{\op{Rate}_{\pib}(N,\beta/K) \cdot \op{Rate}_f^{\mathfrak{c}}(N, \beta/K)}{\piblb^2} \Bigg)
	\\&\leq \frac{2112 \alphaub \sqrt{K}}{\wlb \piblb \sqrt{N}} \prns{\kappa(\Pi) + \frac{\alphaub}{\alphalb^2}} + \frac{40\alphaub \sqrt{K}\log^{1/2}(K/\beta)}{\wlb\piblb\sqrt{N}} + \frac{4\alphaub}{\wlb\piblb^2} \prns{\op{Rate}_{\pib}(N,\beta/K) \cdot \op{Rate}_f^{\mathfrak{c}}(N, \beta/K)}
\end{align*}
w.p. at least $1-6\beta$, where we invoked \cref{lm:oracle_dr_close_to_truth,lm:estimated_dr_close_to_oracle_dr} to bound the two supremum terms.
\end{proof}

We now build towards the proofs for \cref{lm:oracle_dr_close_to_truth,lm:estimated_dr_close_to_oracle_dr}.
First, we show that assuming $\alphalb > 0$, we can uniformly bound the Lipschitz constant of the functions $\{\alpha \mapsto \exp(-r/\alpha), r \in [0, 1]\}$ by $L := 1/\alphalb^2$. 
\begin{lemma}\label{lm:covering_number_for_w}
	Suppose \cref{asm:standard_cb,asm:alpha_bounded_from_zero}. 
	Let $t \in (0, 1)$, and let $s, a, r$ be fixed (not random variables). Then, the following deterministic functions of $\alpha$, restricted to $[\alphalb, \infty)$, are Lipschitz with Lipschitz constant upper bounded by $1/\alphalb^2$.
	\begin{align*}
	    &\alpha \mapsto \exp(-r/\alpha)
	    \\&\alpha \mapsto \Eb{ \exp(-R/\alpha) \mid S=s, A=a }
	    \\&\alpha \mapsto \Eb{ \exp(-R/\alpha) \mid S=s, A=\pi(s)}
	\end{align*}
\end{lemma}
\begin{proof}
    The first function $\alpha \mapsto \exp(-r/\alpha)$ is continuous and has derivative $\frac{r}{\alpha^2}\exp(-r/\alpha)$.
    Since we're restricting to $[\alphalb, \infty)$, the derivative is upper bounded by $1/\alphalb^2$, which implies that the Lipschitz constant is also bounded by $1/\alphalb^2$.
    For the second and third functions, limits can be passed into the expectation using Dominated Convergence Theorem, since the derivative of the random variable is bounded. Hence, the same reasoning shows that their Lipschitz constant is also upper bounded by $1/\alphalb^2$.
\end{proof}
Note the above lemma also implies that the estimated continuum nuisance in \cref{sec: continuum-est}, as a function of $\alpha$, is also Lipschitz, with Lipschitz constant upper bounded by the same quantity.
This is because the estimated nuisance, as a function of $\alpha$, is a convex combination of functions whose Lipschitz constant is upper bounded by $1/\alphalb^2$.

Now we show a key lemma that uniformly concentrates over both $\Pi$ and $[\alphalb, \alphaub]$. 
\begin{lemma}\label{lm:uniform_true_w_close_to_oracle_dr_w}
	Suppose \cref{asm:standard_cb,asm:alpha_bounded_from_zero}.
	Then, for any $\beta \in (0, 1)$, w.p. $1-\beta$ we have,
	\begin{align*}
		\supoverpi \supoveralpha \abs{W_{DR}(\pi,\alpha) - W(\pi,\alpha)} \leq 
		\frac{288}{\piblb \sqrt{N}} \prns{\kappa(\Pi) + \frac{\alphaub}{2\alphalb^2} \vee 1} + \frac{4}{\piblb \sqrt{N}} \log^{1/2}(1/\beta).
	\end{align*}
\end{lemma}
\begin{proof}
	It is sufficient (and necessary, see Page 108 of \citep{wainwright2019high}) to bound the Rademacher complexity of
	\begin{align*}
		\mF_{\Pi,\alpha} = \Bigg\{ 
		&w_{\pi,\alpha}(s,a,r) \mapsto
		\frac{\pi(a \mid s)}{\pib(a \mid s)} \prns{\exp(-r/\alpha) - \Eb{\exp(-R/\alpha) \mid S=s, A=a}} 
		\\&+ \Eb{\exp(-R/\alpha) \mid S=s,A=\pi(s)}
		\Bigg| \pi \in \Pi, \alpha \in [\alphalb, \alphaub]
		\Bigg\}
	\end{align*}
	This class is strictly larger than what was considered in \citep{athey2021policy,zhou2018offline}, since it is also indexed by the dual variable $\alpha$.
	
	First, notice that these functions are uniformly bounded, since $\exp(-r/\alpha) \in (0, 1]$:
	\begin{align*}
	    &\abs{\frac{\pi(a \mid s)}{\pib(a \mid s)} \prns{\exp(-r/\alpha) - \Eb{\exp(-R/\alpha) \mid S=s, A=a}} + \Eb{\exp(-R/\alpha) \mid S=s,A=\pi(s)}}
	    \\&\leq \piblb^{-1} \abs{ \exp(-r/\alpha) - \Eb{\exp(-R/\alpha) \mid S=s, A=a} } + \Eb{\exp(-R/\alpha) \mid S=s,A=\pi(s)}
	    \\&\leq 2\eta^{-1}
	\end{align*}
	
	We now construct covers in $\|\cdot\|_{L_2(\PP_N)}$ to bound the Rademacher complexity.
	Let $\pi,\wt\pi \in \Pi$ and $\alpha,\wt\alpha \in [\alphalb, \alphaub]$.
	Two useful bounds that we'll use are:
	\begin{enumerate}[label=(\alph*)]
		\item We can bound the $L_2(\PP_N)$ distance between policies by the hamming distance:
		\begin{align*}
			\| \pi(a \mid s) - \wt\pi(a \mid s) \|_{L_2(\PP_N)}^2
			&= \frac{1}{N} \sum_{i=1}^N \prns{ \pi(a \mid s) - \wt\pi(a \mid s) }^2 \\
			&\leq \frac{1}{N} \sum_{i=1}^N \I{\pi(s) \neq \wt\pi(s)} \\
			&= d_H(\pi, \wt\pi)
		\end{align*}
		\item By \cref{lm:covering_number_for_w}, we can bound the $L_2(\PP_N)$ distance between $\exp(-r/\alpha)$ functions 
		$$
		\| \exp(-r/\alpha) - \exp(-r/\wt\alpha) \|_{L_2(\PP_N)} \leq L |\alpha-\wt\alpha|, \text{where } L := 1/\alphalb^2.
		$$
	\end{enumerate}
	
	By triangle inequality, we can separately consider three terms:
	\begin{align*}
		&\nmPn{w_{\pi,\alpha} - w_{\wt\pi, \wt\alpha}} 
		\\&\leq \frac{1}{\piblb} \Bigg( 
		\nmPn{ \pi(a_i \mid s_i)\exp(-r_i/\alpha) - \wt\pi(a_i \mid s_i)\exp(-r_i/\wt\alpha) }  \\
		&+ \nmPn{ \pi(a_i \mid s_i)\Eb{\exp(-R/\alpha) \mid S=s_i,A=a_i} - \wt\pi(a_i \mid s_i) \Eb{\exp(-R/\wt\alpha) \mid S=s_i,A=a_i} } \Bigg) \\
		&+ \nmPn{ \Eb{\exp(-R/\alpha) \mid S=s_i, A=\pi(s_i)} - \Eb{\exp(-R/\wt\alpha) \mid S=s_i, A=\wt\pi(s_i)} }
	\end{align*}
	
	Bound the first term:
	\begin{align*}
		&\nmPn{ \pi(a_i \mid s_i)\exp(-r_i/\alpha) - \wt\pi(a_i \mid s_i)\exp(-r_i/\wt\alpha) } \\
		&\leq \nmPn{ (\pi(a_i \mid s_i) - \wt\pi(a_i \mid s_i)) \exp(-r_i/\alpha) } 
		+ \nmPn{ \wt\pi(a_i \mid s_i) \prns{ \exp(-r_i/\alpha) - \exp(-r_i/\wt\alpha) } } \\
		&\leq \nmPn{ \pi(a_i \mid s_i) - \wt\pi(a_i \mid s_i) } 
		+ \nmPn{ \exp(-r_i/\alpha) - \exp(-r_i/\wt\alpha) } \\
		&\leq \sqrt{d_H(\pi,\tilde\pi)} + L|\alpha-\wt\alpha|
	\end{align*}
	
	Bound the second term:
	\begin{align*}
		&\nmPn{ \pi(a_i \mid s_i)\Eb{\exp(-R/\alpha) \mid S=s_i,A=a_i} - \wt\pi(a_i \mid s_i) \Eb{\exp(-r/\wt\alpha) \mid S=s_i,A=a_i} } \\
		&\leq \nmPn{ (\pi(a_i \mid s_i) - \wt\pi(a_i \mid s_i)) \Eb{\exp(-R/\alpha) \mid S=s_i,A=a_i} } \\
		&+ \nmPn{ \wt\pi(a_i|s_i) \prns{ \Eb{\exp(-R/\alpha) \mid S=s_i,A=a_i} - \Eb{\exp(-R/\wt\alpha) \mid S=s_i,A=a_i} } } \\
		&\leq \nmPn{ \pi(a_i \mid s_i) - \wt\pi(a_i \mid s_i) } 
		+ \nmPn{ \Eb{\exp(-R/\alpha) \mid S=s_i,A=a_i} - \Eb{\exp(-R/\wt\alpha) \mid S=s_i,A=a_i} } \\
		&\leq \sqrt{d_H(\pi,\wt\pi)} + L|\alpha-\wt\alpha|
	\end{align*}
	
	Bound the third term:
	\begin{align*}
		&\nmPn{\Eb{\exp(-R/\alpha) \mid S=s_i, A=\pi(s_i)} - \Eb{\exp(-R/\wt\alpha) \mid S=s_i, A=\wt\pi(s_i)}}
		\intertext{Since $L_2(\PP_N)$ is bounded by $L_\infty(\PP_N)$, and apply triangle inequality to each action, }
		&\leq \max_{i \in [N]} \sum_{a \in \mA} \abs{\pi(a \mid s_i) - \wt\pi(a \mid s_i)} \cdot \abs{ \Eb{\exp(-R/\alpha) \mid S=s_i, A=a} - \Eb{\exp(-R/\wt\alpha) \mid S=s_i, A=a} }
		\\&\leq L|\alpha-\wt\alpha| \max_{i \in [N]} \nm{\pi(s_i) - \wt\pi(s_i)}_1
		\\&\leq 2L|\alpha-\wt\alpha|.
	\end{align*}
	
	Altogether, we have that,
	\begin{align*}
		\nmPn{w_{\pi,\alpha} - w_{\wt\pi, \wt\alpha}}
		&\leq \frac{2}{\piblb} \prns{ \sqrt{d_H(\pi,\wt\pi)} + L|\alpha-\wt\alpha| } + 2L|\alpha-\wt\alpha|
		\\&\leq \frac{3}{\piblb} \prns{ \sqrt{d_H(\pi,\wt\pi)} + L|\alpha-\wt\alpha| }
	\end{align*}
	To bound it by $t$, we can take $d_H(\pi,\wt\pi) \leq \prns{\frac{t\piblb}{6}}^2$ and $|\alpha-\wt\alpha| \leq \frac{t\piblb}{6L}$.
	Since $\alpha \in [\alphalb,\alphaub]$, the covering for $\alpha$ can be done in $\frac{3L(\alphaub-\alphalb)}{t\piblb}$ points.
	By Dudley's chaining (see (5.48) of \citep{wainwright2019high}), we have
	\begin{align*}
		\mR_n(\mF_{\Pi,\alpha}) 
		&\leq \frac{24}{\sqrt{N}} \int_0^{4\piblb^{-1}} \log^{1/2}\prns{ \mN_H((t\piblb/6)^2, \Pi) \cdot \frac{3L(\alphaub-\alphalb)}{t\piblb} } \diff t \\
		&\leq \frac{144}{\piblb \sqrt{N}} \prns{ \int_0^1 \log^{1/2} \mN_H(t^2, \Pi) + \log^{1/2} \frac{L(\alphaub-\alphalb)}{2t} \diff t } \\
		&\leq \frac{144}{\piblb \sqrt{N}} \prns{ \kappa(\Pi) + L\alphaub \vee 1 }
	\end{align*}
	By Theorem 4.10 of \citep{wainwright2019high}, w.p. at least $1-\beta$,
	\begin{align}
		\supoverpi \supoveralpha \abs{W_{DR}(\pi,\alpha) - W(\pi,\alpha)}
		&\leq \frac{288}{\piblb \sqrt{N}} \prns{\kappa(\Pi) + L\alphaub \vee 1} + \frac{4}{\piblb \sqrt{N}} \log^{1/2}(1/\beta).
	\end{align}
\end{proof}

Both \cref{lm:oracle_dr_close_to_truth,lm:estimated_dr_close_to_oracle_dr} will start with the following lemma,
\begin{lemma}\label{lm:merge_sup}
	Let $f, g: \Rl^+ \to \Rl^+$ be functions, then,
	\begin{align}
		\abs{ \sup_{\alpha} \braces{-\alpha \log f(\alpha) - \alpha \delta} - \sup_{\alpha} \braces{-\alpha \log g(\alpha) - \alpha \delta} } \leq \sup_{\alpha} \abs{\alpha \log \prns{1 + \frac{f(\alpha)-g(\alpha)}{g(\alpha)}}}.
	\end{align}
\end{lemma}
\begin{proof}
    Merge the two $\sup$'s together, 
	\begin{align*}
		\abs{ \sup_{\alpha} \braces{-\alpha \log f(\alpha) - \alpha \delta} - \sup_{\alpha} \braces{-\alpha \log g(\alpha) - \alpha \delta} }
		\leq \abs{ \sup_{\alpha} -\alpha \log f(\alpha) + \alpha \log g(\alpha) }
		\leq \sup_{\alpha} \abs{\alpha \log \prns{\frac{f(\alpha)}{g(\alpha)}} }
	\end{align*}
\end{proof}

Compared to the non-distributionally robust setting studied by \citet{zhou2018offline,athey2021policy}, the distributionally robust objective has two additional challenges:
\begin{enumerate}
	\item The empirical process term is not simply the reward, but the log of the moment generating function.
	\item There is an additional supremum over $\alpha$.
\end{enumerate}
We now show that DR with oracle nuisances, i.e. $W^{DR}$, is uniformly close to the ground truth, i.e. $W$.
\begin{lemma}\label{lm:oracle_dr_close_to_truth}
Suppose \cref{asm:standard_cb,asm:alpha_bounded_from_zero,asm:sufficiently_large_n}.
Then, for any $\beta \in (0, 1)$, w.p. $1-\beta$, we have
\begin{align*}
	\supoverpi \abs{\droValue^{DR}(\pi) - \droValue(\pi)} \leq \frac{\alphaub}{\wlb} \prns{ \frac{288}{\piblb \sqrt{N}} \prns{\kappa(\Pi) + \frac{\alphaub}{2\alphalb^2} } + \frac{4}{\piblb \sqrt{N}} \log^{1/2}(1/\beta) }.
\end{align*}
\end{lemma}
\begin{proof}[Proof of \cref{lm:oracle_dr_close_to_truth}]
	By \cref{lm:merge_sup}, 
	\begin{align*}
		\supoverpi \abs{\droValue^{DR}(\pi) - \droValue(\pi)} \leq \supoverpi \supoveralpha \abs{\alpha \log \prns{1 + X(\pi,\alpha)}}, \text{where } X(\pi,\alpha) := \frac{W_{DR}(\pi,\alpha) - W(\pi,\alpha)}{W(\pi, \alpha)}
	\end{align*}
	First, due to \cref{asm:sufficiently_large_n}, w.p. at least $1-\beta$, we have $\supoverpi \supoveralpha \abs{X(\pi,\alpha)} < 1/2$. This is because the denominator is lower bounded by $\wlb$ by \cref{lm:w_lower_bound}. Then, \cref{lm:uniform_true_w_close_to_oracle_dr_w} implies the numerator is bounded w.h.p. by $\wlb/2$.
	Hence, under this high probability event, the above expression is well-defined.
	
	Finally, conditioning on this high-probability event, and using $\abs{\log(1+x)} \leq \abs{x}$ if $\abs{x} < 0.8$, we have,
	\begin{align*}
		\supoverpi \supoveralpha \abs{\alpha \log \prns{1 + X(\pi,\alpha)}}
		&\leq \supoverpi \supoveralpha \abs{\alpha X(\pi,\alpha)}
		\\&\leq \frac{\alphaub}{\wlb} \prns{ \frac{288}{\piblb \sqrt{N}} \prns{\kappa(\Pi) + \frac{\alphaub}{2\alphalb^2} \vee 1} + \frac{4}{\piblb \sqrt{N}} \log^{1/2}(1/\beta) }.
	\end{align*}
\end{proof}

\begin{lemma}\label{lm:estimated_dr_close_to_oracle_dr}
Suppose
\cref{asm:standard_cb,asm:alpha_bounded_from_zero,asm:sufficiently_large_n}. 
Then, for any $\beta \in (0, 1)$, w.p. $1-5\beta$, we have 
\begin{align*}
   \supoverpi  \abs{\droValue^{DR}(\pi) - \est{\droValue}^{DR}(\pi)} 
   &\leq
   \frac{2\alphaub}{\wlb} \prns{ \frac{384}{\piblb\sqrt{N/K}} \Bigg( \kappa(\Pi) + \frac{\alphaub}{\alphalb^2} } + \frac{8\log^{1/2}(K/\beta)}{\piblb\sqrt{N/K}} 
	+ \frac{\op{Rate}_{\pib}(N,\beta/K) \cdot \op{Rate}_f^{\mathfrak{c}}(N, \beta/K)}{\piblb^2} \Bigg).
\end{align*}
\end{lemma}
\begin{proof}[Proof of \cref{lm:estimated_dr_close_to_oracle_dr}]
	By \cref{lm:merge_sup},
	\begin{align*}
		\sup_{\pi \in \Pi} \abs{\droValue^{DR}(\pi) - \est{\droValue}^{DR}(\pi)} 
		\leq \supoverpi \supoveralpha \abs{\alpha \log \prns{1 + Y(\pi,\alpha)}}, \text{where } Y(\pi,\alpha) := \frac{W^{DR}(\pi,\alpha)-\wh W^{DR}(\pi,\alpha)}{W^{DR}(\pi,\alpha)}
	\end{align*}
	
	Decompose the numerator of $Y$ as follows, which is only possible due to the doubly robust structure,
	\begin{align*}
		&\wh W^{DR}(\pi,\alpha) - W^{DR}(\pi,\alpha) 
		\\&= \frac{1}{N} \sum_{k = 1}^K \sum_{i \in \mI_k} \prns{\E_{a \sim \pi(s_i)}\bracks{\wh f_0^{(k)}(s_i, a; \alpha) - f_0(s_i, a; \alpha)} - \frac{\pi(a_i \mid s_i)}{\pib(a_i \mid s_i)} \prns{\wh f_0^{(k)}(s_i, a_i; \alpha) - f_0(s_i, a_i; \alpha)}}
		\\&+ \frac{1}{N} \sum_{k=1}^K \sum_{i \in \mI_k} \prns{\frac{\pi(a_i \mid s_i)}{\wh\pib^{(k)}(a_i \mid s_i)} - \frac{\pi(a_i \mid s_i)}{\pib(a_i \mid s_i)}} \prns{\exp(-r_i/\alpha) - f_0(s_i, a_i; \alpha)}
		\\&+ \frac{1}{N} \sum_{k=1}^K \sum_{i \in \mI_k} \prns{\frac{\pi(a_i \mid s_i)}{\wh\pib^{(k)}(a_i \mid s_i)} - \frac{\pi(a_i \mid s_i)}{\pib(a_i \mid s_i)}} \prns{f_0(s_i, a_i; \alpha) - \wh f_0^{(k)}(s_i, a_i; \alpha)}
		\\&= \frac{1}{K} \sum_{k = 1}^K \mE_1(\pi,\alpha,k) + \mE_2(\pi,\alpha,k) + \mE_3(\pi,\alpha,k)
	\end{align*}
	where
	\begin{align}
		&\mE_1(\pi,\alpha,k) := \frac{1}{|\mI_k|} \sum_{i \in \mI_k} \prns{\E_{a \sim \pi(s_i)}\bracks{\wh f_0^{(k)}(s_i, a; \alpha) - f_0(s_i, a; \alpha)} - \frac{\pi(a_i \mid s_i)}{\pib(a_i \mid s_i)} \prns{\wh f_0^{(k)}(s_i, a_i; \alpha) - f_0(s_i, a_i; \alpha)}} \label{eq:estimated_dr_decomp_e1} \\
		&\mE_2(\pi,\alpha,k) := \frac{1}{|\mI_k} \sum_{i \in \mI_k} \prns{\frac{\pi(a_i \mid s_i)}{\wh\pib^{(k)}(a_i \mid s_i)} - \frac{\pi(a_i \mid s_i)}{\pib(a_i \mid s_i)}} \prns{\exp(-r_i/\alpha) - f_0(s_i, a_i; \alpha)} \label{eq:estimated_dr_decomp_e2} \\
		&\mE_3(\pi,\alpha,k) := \frac{1}{|\mI_k|} \sum_{i \in \mI_k} \prns{\frac{\pi(a_i \mid s_i)}{\wh\pib^{(k)}(a_i \mid s_i)} - \frac{\pi(a_i \mid s_i)}{\pib(a_i \mid s_i)}} \prns{f_0(s_i, a_i; \alpha) - \wh f_0^{(k)}(s_i, a_i; \alpha)} \label{eq:estimated_dr_decomp_e3}
	\end{align}
	A key observation is that $\wh f_0^{(k)}(\cdot, \cdot; \alpha)$ and $\wh\pib^{(k)}(\cdot \mid \cdot)$ are constant on the fold they are evaluated. In other words, in $\mD[\mI_k]$, $\wh f_0^{(k)}(s_i, a_i; \alpha)$ and $\wh\pib^{(k)}(a_i \mid s_i)$ are only functions of the current points $s_i, a_i$, and are independent from every other summand in the current fold (but they are not independent from the summands on folds that they were fitted on!).
	Then, each $\mE_i(\pi,\alpha,k)$ is a sum of i.i.d. random variables, and in particular zero mean random variables.
	This is why cross-fitting is crucial --- if we didn't cross-fit, $\wh f_0^{(k)}(s_i,a_i;\alpha)$ and $\wh\pib^{(k)}(a_i \mid s_i)$ would also be functions of the rest of the dataset, which precludes the convenient independence property.
	\cref{lm:estimated_dr_decomp_e1_bound,lm:estimated_dr_decomp_e2_bound} provide bounds for $\mE_1, \mE_2$ using similar Rademacher complexity arguments.
	The error term $\mE_3$ is a product of estimation errors, which we bound directly in \cref{lm:estimated_dr_decomp_e3_bound} with the estimation rates.
	Putting this together gives a bound on the numerator of $Y$: w.p. at least $1-4\beta$:
	\begin{align*}
	    &\supoverpi \supoveralpha \abs{\wh W^{DR}(\pi,\alpha) - W^{DR}(\pi,\alpha) } 
	    \\&\leq \frac{1}{K} \sum_{k=1}^K \abs{ \mE_1(\pi,\alpha,k) } + \abs{ \mE_2(\pi,\alpha,k) } + \abs{ \mE_3(\pi,\alpha,k) }
	    \\&\leq \frac{384}{\piblb\sqrt{N/K}} \prns{ \kappa(\Pi) + L\alphaub \vee 1 } + \frac{8\log^{1/2}(K/\beta)}{\piblb\sqrt{N/K}} + \frac{\op{Rate}_{\pib}(N,\beta/K) \cdot \op{Rate}_f^{\mathfrak{c}}(N, \beta/K)}{\piblb^2}
	\end{align*}
	where we assumed $K$ divides $N$, so $|\mI_k| = N/K$ for convenience.

    We now lower bound the worst-case denominator of $Y$,
    \begin{align*}
        \inf_{\pi \in \Pi} \inf_{\alpha \in [\alphalb, \alphaub]} |W^{DR}(\pi,\alpha)|
        &\geq \inf_{\pi \in \Pi} \inf_{\alpha \in [\alphalb, \alphaub]} |W(\pi,\alpha)| - |W^{DR}(\pi,\alpha) - W(\pi,\alpha)|
        \\&\geq \inf_{\pi \in \Pi} \inf_{\alpha \in [\alphalb, \alphaub]} |W(\pi,\alpha)| - \supoverpi \supoveralpha |W^{DR}(\pi,\alpha) - W(\pi,\alpha)|
        \\&\geq \wlb - \prns{ \frac{288}{\piblb \sqrt{N}} \prns{\kappa(\Pi) + \frac{\alphaub}{2\alphalb^2} \vee 1} + \frac{4}{\piblb \sqrt{N}} \log^{1/2}(1/\beta) },
    \end{align*}
    where the last inequality holds w.p. at least $1-\beta$, due to \cref{lm:w_lower_bound,lm:uniform_true_w_close_to_oracle_dr_w}.
    Our assumption on $N$ being sufficiently large (\cref{asm:sufficiently_large_n}) implies that the subtracted term is at most $\wlb/2$.
    So, the worst-case denominator of $Y$ is lower bounded by $\wlb/2$. 
	
	Putting the two bounds together, we can bound the worst-case $Y$: w.p. at least $1-5\beta$, 
	\begin{align*}
	    \supoverpi \supoveralpha \abs{Y(\pi,\alpha)} 
	    &\leq \frac{2}{\wlb} \Bigg( \frac{384}{\piblb\sqrt{N/K}} \prns{ \kappa(\Pi) + L\alphaub \vee 1 } + \frac{8\log^{1/2}(K/\beta)}{\piblb\sqrt{N/K}} + \frac{\op{Rate}_{\pib}(N,\beta/K) \cdot \op{Rate}_f^{\mathfrak{c}}(N, \beta/K)}{\piblb^2} \Bigg),
	\end{align*}
	which is at most $1/2$ when $N$ is sufficiently large (\cref{asm:sufficiently_large_n}). 
	Since $|\log(1+x)| \leq |x|$ when $|x| < 0.8$, 
	we have that,
	\begin{align*}
		&\supoverpi \abs{\droValue^{DR}(\pi) - \est{\droValue}^{DR}(\pi)}
		\\&\leq \supoverpi \supoveralpha \abs{\alpha \log \prns{1 + Y(\pi,\alpha)}}
		\\&\leq \supoverpi \supoveralpha \abs{\alpha Y(\pi,\alpha)}
		\\&\leq \frac{2\alphaub}{\wlb} \prns{ \frac{384}{\piblb\sqrt{N/K}} \prns{ \kappa(\Pi) + L\alphaub \vee 1 } + \frac{8\log^{1/2}(K/\beta)}{\piblb\sqrt{N/K}} 
		+ \frac{\op{Rate}_{\pib}(N,\beta/K) \cdot \op{Rate}_f^{\mathfrak{c}}(N, \beta/K)}{\piblb^2} }
	\end{align*}
	which concludes the proof.
\end{proof}

\begin{lemma}\label{lm:estimated_dr_decomp_e1_bound}
    Suppose \cref{asm:standard_cb,asm:alpha_bounded_from_zero}. 
    Then, for any $\beta \in (0, 1)$, w.p. $1-\beta$, we have,
    \begin{align*}
        \forall k \in [K]: \supoverpi \supoveralpha \abs{\mE_1(\pi,\alpha,k)} \leq \frac{192}{\piblb\sqrt{|\mI_k|}} \prns{ \kappa(\Pi) + L\alphaub \vee 1 } + \frac{4\log^{1/2}(K/\beta)}{\piblb\sqrt{|\mI_k|}},
    \end{align*}
    where $\mE_1$ is defined in \cref{eq:estimated_dr_decomp_e1}.
\end{lemma}
\begin{proof}
    Let $k \in [K]$ be fixed for now.
    Each summand of $\mE_1(\cdot, \cdot, k)$ is zero-mean, since importance sampling is unbiased.
    We now bound the Rademacher complexity of 
	\begin{align*}
		\mF := \braces{ 
			(s,a) \mapsto \Eb[\bar a \sim \pi(s)]{\wh f_0^{(k)}(s, \bar a; \alpha) - f_0(s, \bar a; \alpha)} - \frac{\pi(a \mid s)}{\pib(a \mid s)} \prns{ \wh f_0^{(k)}(s, a; \alpha) - f_0(s, a; \alpha) } \Bigg| \pi \in \Pi, \alpha \in [\alphalb, \alphaub] } \\
	\end{align*}
	First, we bound the envelope,
	\begin{align*}
	    &\abs{ \Eb[\bar a \sim \pi(s)]{\wh f_0^{(k)}(s, \bar a; \alpha) - f_0(s, \bar a; \alpha)} - \frac{\pi(a \mid s)}{\pib(a \mid s)} \prns{ \wh f_0^{(k)}(s, a; \alpha) - f_0(s, a; \alpha) } } 
	    \\&\leq 1+\piblb^{-1} \leq 2\piblb^{-1}
	\end{align*}
	
	Now we cover in $L_2(\PP_{\mI_k})$ (empirical distribution on $\mD[\mI_k]$). So let $\pi, \wt\pi \in \Pi$ and $\alpha, \wt\alpha \in [\alphalb, \alphaub]$, then
	\begin{align*}
	    &\sqrt{ \frac{1}{|\mI_k|}\sum_{i \in \mI_k} \prns{ \Eb[a \sim \pi(s_i)]{\wh f_0^{(k)}(s_i, a; \alpha) - f_0(s_i, a; \alpha)} - \Eb[a \sim \wt\pi(s_i)]{\wh f_0^{(k)}(s_i, a; \wt\alpha) - f_0(s_i, a; \wt\alpha)} }^2 }
	    \\&\leq \max_{i \in \mI_k]} \sum_{a \in \mA} \abs{\pi(a \mid s_i) - \wt\pi(a \mid s_i)}\abs{ \wh f_0^{(k)}(s_i, a; \alpha) - f_0(s_i, a; \alpha) - (\wh f_0^{(k)}(s_i, a; \wt\alpha) - f_0(s_i, a; \wt\alpha))}
	    \\&\leq \max_{i \in \mI_k} \|\pi(s_i)-\wt\pi(s_i)\|_1 \cdot L|\alpha-\wt\alpha|
	    \\&\leq 2L|\alpha-\wt\alpha|,
	\end{align*}
	and
	\begin{align*}
	    &\sqrt{ \frac{1}{|\mI_k|}\sum_{i \in \mI_k} \prns{ \frac{\pi(a_i \mid s_i)}{\pib(a_i \mid s_i)} 
			\prns{ \wh f_0^{(k)}(s_i,a_i;\alpha) - f_0(s_i,a_i;\alpha) } - \frac{\wt\pi(a_i \mid s_i)}{\pib(a_i \mid s_i)} 
			\prns{ \wt f_0^{(k)}(s_i,a_i;\wt\alpha) - f_0(s_i,a_i;\wt\alpha) } }^2 }
		\\&\leq \piblb^{-1} \sqrt{ \frac{1}{|\mI_k|}\sum_{i \in \mI_k} \prns{ (\pi(a_i\mid s_i) - \wt\pi(a_i\mid s_i))(\wh f_0^{(k)}(s_i,a_i;\alpha) - f_0(s_i,a_i;\alpha) }^2 }
		\\&+ \piblb^{-1} \sqrt{ \frac{1}{|\mI_k|}\sum_{i \in \mI_k} \prns{ \wt\pi(a_i\mid s_i) \prns{ \wh f_0^{(k)}(s_i,a_i;\alpha) - f_0(s_i,a_i;\alpha) - (\wh f_0^{(k)}(s_i,a_i;\wt\alpha)-f_0(s_i,a_i;\wt\alpha)) } }^2 }
		\\&\leq \piblb^{-1} \sqrt{d_H(\pi,\wt\pi)} + \piblb^{-1} 2L|\alpha-\wt\alpha|
	\end{align*}
	Combining the two bounds, we get that the total bound is at most $\piblb^{-1} \sqrt{d_H(\pi,\wt\pi)} + 4\piblb^{-1}|\alpha-\wt\alpha|$, so for any $t$, we can make $d_H(\pi,\wt\pi) \leq (t/2\piblb)^2$ and $|\alpha-\wt\alpha| \leq t/8L\piblb$ to bound by $t$.
	By (5.48) of \citep{wainwright2019high}, we have
	\begin{align*}
		\mR_N(\mF) 
		&\leq \frac{24}{\sqrt{|\mI_k|}} \int_0^{4\piblb^{-1}} \log^{1/2}\prns{\mN_H((t/2\piblb)^2, \Pi) \cdot \frac{4L(\alphaub-\alphalb)}{t\piblb}} \diff t
		\\&\leq \frac{96}{\piblb\sqrt{|\mI_k|}} \prns{ \kappa(\Pi) + L\alphaub \vee 1 }
	\end{align*}
	By Theorem 4.10 of \citep{wainwright2019high}, w.p. $1-\beta$,
	\begin{align*}
		\supoverpi \supoveralpha \abs{\mE_1(\pi,\alpha,k)}
		&\leq \frac{192}{\piblb\sqrt{|\mI_k|}} \prns{ \kappa(\Pi) + L\alphaub \vee 1 } + \frac{4\log^{1/2}(1/\beta)}{\piblb\sqrt{|\mI_k|}}.
	\end{align*}
	Union bound over $k$ yields the result.
\end{proof}

\begin{lemma}\label{lm:estimated_dr_decomp_e2_bound}
    Suppose \cref{asm:standard_cb,asm:alpha_bounded_from_zero}. 
    Then, for any $\beta \in (0, 1)$, w.p. $1-\beta$, we have,
    \begin{align*}
         \forall k \in [K]: \supoverpi \supoveralpha \abs{\mE_2(\pi,\alpha,k)} \leq \frac{192}{\piblb\sqrt{|\mI_k|}} \prns{\kappa(\Pi) + L\alpha \vee 1} + \frac{4\log^{1/2}(K/\beta)}{\piblb\sqrt{\mI_k}},
    \end{align*}
    where $\mE_2$ is defined in \cref{eq:estimated_dr_decomp_e2}.
\end{lemma}
\begin{proof}
	Let $k \in [K]$ be fixed for now. 
	Each summand of $\mE_2$ is zero-mean due to the definition of $f_0$.
	We now bound the Rademacher complexity of
	\begin{align*}
		\mF &= \braces{ (s, a, r) \mapsto \prns{\frac{\pi(a \mid s)}{\wh\pib^{(k)}(a \mid s)} - \frac{\pi(a \mid s)}{\pib(a \mid s)}} (\exp(-r/\alpha) - f_0(s_i,a_i;\alpha)) \Bigg| \pi \in \Pi, \alpha \in [\alphalb, \alphaub] } \\
	\end{align*}
	First, we bound the envelope,
	\begin{align*}
	    \abs{ \prns{\frac{\pi(a \mid s)}{\wh\pib^{(k)}(a \mid s)} - \frac{\pi(a \mid s)}{\pib(a \mid s)}} (\exp(-r/\alpha) - f_0(s,a;\alpha)) }
	    &\leq 2\piblb^{-1}.
	\end{align*}
	
	Now, we cover in $L_2(\PP_{\mI_k})$ (empirical distribution on $\mD[\mI_k]$).
	So let $\pi,\wt\pi \in \Pi, \alpha,\wt\alpha \in [\alphalb, \alphaub]$, then
	\begin{align*}
	    &\sqrt{ \frac{1}{|\mI_k|}\sum_{i \in \mI_k} \prns{ \prns{\frac{\pi(a_i \mid s_i)}{\wh\pib^{(k)}(a_i\mid s_i)} - \frac{\pi(a_i\mid s_i)}{\pib(a_i\mid s_i)}} \exp(-r_i/\alpha) - \prns{\frac{\wt\pi(a_i \mid s_i)}{\wh\pib^{(k)}(a_i\mid s_i)} - \frac{\wt\pi(a_i\mid s_i)}{\pib(a_i\mid s_i)}} \exp(-r_i/\wt\alpha) }^2 }
	    \\&\leq 2\piblb^{-1} \sqrt{ \frac{1}{|\mI_k|}\sum_{i \in \mI_k} \prns{ \pi(a_i \mid s_i) \exp(-r_i/\alpha) - \wt\pi(a_i \mid s_i)\exp(-r_i/\wt\alpha) }^2 }
	    \\&\leq 2\piblb^{-1} \sqrt{ \frac{1}{|\mI_k|}\sum_{i \in \mI_k} \prns{ (\pi(a_i\mid s_i)-\wt\pi(a_i\mid s_i)) \exp(-r_i/\alpha) }^2 }
	    \\&+ 2\piblb^{-1} \sqrt{ \frac{1}{|\mI_k|}\sum_{i \in \mI_k} \prns{ \wt\pi(a_i\mid s_i)(\exp(-r_i/\alpha) - \exp(-r_i/\wt\alpha)) }^2 }
	    \\&\leq 2\piblb^{-1} \sqrt{d_H(\pi,\wt\pi)} + 2\piblb^{-1} L|\alpha-\wt\alpha|
	\end{align*}
	Replacing $\exp(-r_i/\alpha)$ by $f_0(s_i,a_i;\alpha)$ in the above arguments yields the same bound, since $f_0$ is also $L$-Lipschitz in $\alpha$ (\cref{lm:covering_number_for_w}). Thus, the total distance bound is $4\piblb^{-1}(\sqrt{d_H(\pi,\wt\pi)} + L|\alpha-\wt\alpha|)$.
	So for any $t$, we can make $d_H(\pi,\wt\pi) \leq (t/8\piblb)^2$ and $|\alpha-\wt\alpha| \leq t/8L\piblb$ to bound by $t$.
	By (5.48) of \citep{wainwright2019high}, we have
	\begin{align*}
	    \mR_N(\mF) 
	    &\leq \frac{24}{\sqrt{|\mI_k|}} \int_0^{4\piblb^{-1}} \log^{1/2}\prns{ \mN_H((t/8\piblb)^2, \Pi) \cdot \frac{4L(\alphaub-\alphalb)}{t\piblb} } \diff t
	    \\&\leq \frac{96}{\piblb\sqrt{|\mI_k|}} \prns{ \kappa(\Pi) + L\alphaub \vee 1 }.
	\end{align*}
	By Theorem 4.10 of \citep{wainwright2019high}, w.p. $1-\beta$,
	\begin{align*}
		\supoverpi \supoveralpha \abs{\mE_2(\pi,\alpha,k)} \leq \frac{192}{\piblb\sqrt{|\mI_k|}} \prns{\kappa(\Pi) + L\alpha \vee 1} + \frac{4\log^{1/2}(1/\beta)}{\piblb\sqrt{\mI_k}}.
	\end{align*}
	Union bound over $k$ yields the result.
\end{proof}

\begin{lemma}\label{lm:estimated_dr_decomp_e3_bound}
    Suppose \cref{asm:standard_cb,asm:alpha_bounded_from_zero}. 
    Then, for any $\beta \in (0, 1/2)$, w.p. $1-2\beta$, we have,
    \begin{align*}
         \forall k \in [K]: \supoverpi \supoveralpha \abs{\mE_3(\pi,\alpha,k)} \leq \frac{\op{Rate}_{\pib}(N,\beta/K) \cdot \op{Rate}_f^{\mathfrak{c}}(N, \beta/K)}{\piblb^2},
    \end{align*}
    where $\mE_3$ is defined in \cref{eq:estimated_dr_decomp_e3}.
\end{lemma}
\begin{proof}
    Let $k \in [K]$ be fixed first.
	\begin{align*}
	    &\Eb{\supoverpi \supoveralpha \abs{\mE_3(\pi,\alpha,k)}}
		\\&=\Eb{ \supoverpi \supoveralpha \abs{ \frac{1}{|\mI_k|} \sum_{i \in \mI_k} \prns{\frac{\pi(a_i \mid s_i)}{\wh\pib^{(k)}(a_i \mid s_i)} - \frac{\pi(a_i \mid s_i)}{\pib(a_i \mid s_i)}} \prns{\wh f_0^{(k)}(s_i, a_i;\alpha) - f_0(s_i, a_i;\alpha)} } }
		\\&\leq \Eb{\sqrt{ \frac{1}{|\mI_k|} \sum_{i \in \mI_k} \prns{\frac{1}{\wh\pib^{(k)}(a_i \mid s_i)} - \frac{1}{\pib(a_i \mid s_i)}}^2 } \cdot \supoveralpha \sqrt{ \frac{1}{|\mI_k|} \sum_{i \in \mI_k} \prns{\wh f_0^{(k)}(s_i, a_i; \alpha) - f_0(s_i, a_i; \alpha)}^2 } }
		\\&\leq \sqrt{ \Eb{ \frac{1}{|\mI_k|} \sum_{i \in \mI_k} \prns{\frac{1}{\wh\pib^{(k)}(a_i \mid s_i)} - \frac{1}{\pib(a_i \mid s_i)}}^2 } } \cdot \sqrt{ \Eb{\supoveralpha \frac{1}{|\mI_k|} \sum_{i \in \mI_k} \prns{\wh f_0^{(k)}(s_i, a_i; \alpha) - f_0(s_i, a_i; \alpha)}^2 } }
		\intertext{Since $\Eb{\sup \frac{1}{N} \sum (\cdot)} \leq \frac{1}{N} \sum \Eb{\sup (\cdot)}$,  }
		&\leq \sqrt{ \Eb{ \prns{\frac{1}{\wh\pib^{(k)}(A \mid S)} - \frac{1}{\pib(A \mid S)}}^2 } } \cdot \sqrt{ \Eb{\supoveralpha \prns{\wh f_0^{(k)}(S, A; \alpha) - f_0(S, A; \alpha)}^2 } }
		\intertext{Using definition of estimation rates, and the fact that $\wh\pib^{(k)}, \wh f_0^{(k)}(\cdot; \alpha)$ were trained on $N-|\mI_k|=N(1-1/K)$ data points (due to cross-fitting), we have w.p. $1-2\beta$, }
		&\leq \frac{1}{\piblb^2} \op{Rate}_{\pib}(N,\beta) \cdot \op{Rate}_f^{\mathfrak{c}}(N, \beta)
	\end{align*}
	Finally apply union bound over $k$.
\end{proof}

\subsection{Point-wise rate to uniform rate for Lipschitz regressions}

We now show that when the target function is Lipschitz on a compact domain, point-wise rates can be translated into uniform rates.
Let $\wh f(x; \alpha)$ be estimates of $f(x; \alpha)$, where $\alpha \in [0, b]$ for some $b \in \Rl$. 
Supposing that $\wh f$ is learned on a random sample of $N$ datapoints, we define the point-wise convergence rate such that for any $\alpha \in [0,b]$, for any $\beta \in (0, 1)$, w.p. at least $1-\beta$, we have 
\begin{align*}
    \| \wh f(x; \alpha) - f(x; \alpha) \|_{L_2(\PP_0)} \leq \op{Rate}_{point}(N,\beta).
\end{align*}
Define the uniform rate so that for any $\beta$, w.p. $1-\beta$, we have
\begin{align*}
    \|\sup_{\alpha \in [0,b]} \wh f(x;\alpha) - f(x;\alpha)\|_{L_2(\PP_0)} \leq \op{Rate}_{unif}(N,\beta).
\end{align*}
\begin{lemma}\label{lm:pointwise-to-lipschitz}
Suppose $\wh f, f$ are both $L$-Lipschitz. Then, for any $\beta$, w.p. $1-\beta$, we have
\begin{align*}
    \op{Rate}_{unif}(N, \beta) \leq \inf_{0 < d \leq b} 2dL + \op{Rate}_{point}(N, 2d\beta/(b+2d))
\end{align*}
\end{lemma}
\begin{proof}
The idea is to ensure we have point-wise guarantees at the crucial grid points, placed at distance $2d$ apart from each other from $[0,b]$, similar to Lemma C.1 of \citep{oprescu2019orthogonal}.
So there are at most $\ceil{b/2d}$ points. Let $n(\alpha)$ denote the closest grid point, so we have $\abs{\alpha-n(\alpha)} \leq d$ for any $\alpha$.
Also at each grid point, we have $\|\wh f(x;\alpha) - f(x;\alpha)\|_{L_2(\PP_0)} \leq \op{Rate}_{point}(N,\beta/\ceil{b/2d}) \leq \op{Rate}_{point}(N, 2d\beta/(b+2d))$.
Hence, w.p. $1-\beta$,
\begin{align*}
    &\|\sup_{\alpha \in [0,b]} \wh f(x;\alpha) - f(x;\alpha)\|_{L_2(\PP_0)}
    \\&\leq \|\sup_{\alpha \in [0,b]} \abs{ \wh f(x;\alpha) - \wh f(x; n(\alpha)) } + \abs{ \wh f(x; n(\alpha)) - f(x; n(\alpha)) } + \abs{ f(x;n(\alpha) - f(x;\alpha) } \|_{L_2(\PP_0)}
    \\&\leq \|\sup_{\alpha \in [0,b]} \wh f(x;\alpha) - \wh f(x; n(\alpha)) \|_{L_2(\PP_0)} + \|\sup_{\alpha \in [0,b]} \wh f(x; n(\alpha)) - f(x; n(\alpha)) \|_{L_2(\PP_0)} + \| \sup_{\alpha \in [0,b]} f(x;n(\alpha) - f(x;\alpha) \|_{L_2(\PP_0)}
    \\&\leq 2dL + \op{Rate}_{point}(N, 2d\beta/(b+2d))
\end{align*}
\end{proof}
Typically, the point-wise rate guarantees are of the form $\op{Rate}_{point}(N, \beta) = C(\frac{1}{N^p} + \sqrt{\log(1/\beta)/N})$ \citep{wainwright2019high,bartlett2005local}.
In this case, setting $d = \frac{1}{N^p}$ gives the guarantee that 
\begin{align*}
    \op{Rate}_{unif}(N,\beta) \leq \frac{C+2L}{N^p} + \sqrt{\frac{\log(N^p(b+2/N^p)/2\beta)}{N}},
\end{align*}
which is $\mO(\sqrt{\log(N)/N})$ if $p=1/2$ and $\mO(N^{-p})$ if $p < 1/2$. Hence, when the regression target is Lipschitz, the uniform guarantee has the same rate as pointwise if the rate is non-parametric (slower than $\sqrt{N}$). And when the pointwise rate is a parametric $\sqrt{N}$ rate, then the uniform rate only incurs an extra $\sqrt{\log(N)}$ factor.

\end{document}